%% file: neurips_2022.tex
\documentclass{article}

\PassOptionsToPackage{sort, numbers, compress}{natbib}

    \usepackage[final]{neurips_2022}

\usepackage[utf8]{inputenc} 
\usepackage[T1]{fontenc}    
\usepackage{hyperref}       
\usepackage{url}            
\usepackage{booktabs}       
\usepackage{amsfonts}       
\usepackage{nicefrac}       
\usepackage{microtype}      
\usepackage{xcolor}         

\usepackage{fdsymbol}

\usepackage{algorithm}
\usepackage[noend]{algpseudocode}

\usepackage{amsthm}

\usepackage{multirow}
\usepackage{enumitem}

\usepackage[capitalize,noabbrev]{cleveref}
\usepackage{fdsymbol}
\newtheorem{theorem}{Theorem}[section]
\newtheorem{proposition}[theorem]{Proposition}
\newtheorem{lemma}[theorem]{Lemma}

\newtheorem{definition}[theorem]{Definition}
\newtheorem{assumption}[theorem]{Assumption}

\usepackage{multirow}
\usepackage{graphicx}
\usepackage{subfig}

\usepackage{rotating}

\usepackage{tikz}
\usepackage{caption}

\usepackage{makecell}

\title{Counterfactual Fairness with Partially Known \\ Causal Graph}

 \makeatletter
\def\@fnsymbol#1{\ensuremath{\ifcase#1\or \dagger\or \ddagger\or
   \mathsection\or \mathparagraph\or \|\or **\or \dagger\dagger
   \or \ddagger\ddagger \else\@ctrerr\fi}}
    \makeatother

\author{%
  Aoqi Zuo \\
   The University of Melbourne\\
  \texttt{azuo@student.unimelb.edu.au} \\
   \And
   Susan Wei \\
   The University of Melbourne \\
   \texttt{susan.wei@unimelb.edu.au} \\
   \AND
   Tongliang Liu \\
   The University of Sydney \\
   \texttt{tongliang.liu@sydney.edu.au} \\
   \And
   Bo Han \\
   Hong Kong Baptist University \\
   \texttt{bhanml@comp.hkbu.edu.hk} \\
   \And
   Kun Zhang \\
   Carnegie Mellon University \& MBZUAI \\
   \texttt{kunz1@cmu.edu} \\
   \And
   Mingming Gong\thanks{Corresponding author} \\
   The University of Melbourne \\
   \texttt{mingming.gong@unimelb.edu.au} \\
}

\begin{document}


\maketitle

\input{Pages/Abstract}
\input{Pages/Introduction_v2}
\input{Pages/Background_v2}
\input{Pages/Problem}
\input{Pages/Method}

\input{Pages/Experiment}
\input{Pages/Conclusion}
\input{Pages/Acknowledgement}

\clearpage
\bibliographystyle{plain}
\bibliography{reference}

\clearpage
\section*{Checklist}


\begin{enumerate}

\item For all authors...
\begin{enumerate}
  \item Do the main claims made in the abstract and introduction accurately reflect the paper's contributions and scope?
    \answerYes
  \item Did you describe the limitations of your work?
    \answerYes
  \item Did you discuss any potential negative societal impacts of your work?
    \answerNA{}
  \item Have you read the ethics review guidelines and ensured that your paper conforms to them?
    \answerYes
\end{enumerate}

\item If you are including theoretical results...
\begin{enumerate}
  \item Did you state the full set of assumptions of all theoretical results?
    \answerYes
        \item Did you include complete proofs of all theoretical results?
    \answerYes
\end{enumerate}

\item If you ran experiments...
\begin{enumerate}
  \item Did you include the code, data, and instructions needed to reproduce the main experimental results (either in the supplemental material or as a URL)?
    \answerYes
  \item Did you specify all the training details (e.g., data splits, hyperparameters, how they were chosen)?
    \answerYes
        \item Did you report error bars (e.g., with respect to the random seed after running experiments multiple times)?
    \answerYes
        \item Did you include the total amount of compute and the type of resources used (e.g., type of GPUs, internal cluster, or cloud provider)?
    \answerNA{}
\end{enumerate}

\item If you are using existing assets (e.g., code, data, models) or curating/releasing new assets...
\begin{enumerate}
  \item If your work uses existing assets, did you cite the creators?
    \answerYes
  \item Did you mention the license of the assets?
    \answerNA{}
  \item Did you include any new assets either in the supplemental material or as a URL?
    \answerNA{}
  \item Did you discuss whether and how consent was obtained from people whose data you're using/curating?
    \answerNA{}
  \item Did you discuss whether the data you are using/curating contains personally identifiable information or offensive content?
    \answerNA{}
\end{enumerate}

\item If you used crowdsourcing or conducted research with human subjects...
\begin{enumerate}
  \item Did you include the full text of instructions given to participants and screenshots, if applicable?
    \answerNA{}
  \item Did you describe any potential participant risks, with links to Institutional Review Board (IRB) approvals, if applicable?
    \answerNA{}
  \item Did you include the estimated hourly wage paid to participants and the total amount spent on participant compensation?
    \answerNA{}
\end{enumerate}

\end{enumerate}


\clearpage
\appendix
\input{Pages/Appendices/Preliminaries}
\input{Pages/Appendices/Proofs}

\input{Pages/Appendices/SuppMethod}

\input{Pages/Appendices/SuppExp}
\input{Pages/Appendices/RelatedWork}
\input{Pages/Appendices/Discussion}

\end{document}

%% file: Pages/Abstract.tex
\begin{abstract}

Fair machine learning aims to avoid treating individuals or sub-populations unfavourably based on \textit{sensitive attributes}, such as gender and race.
Those methods in fair machine learning that are built on causal inference ascertain discrimination and bias through causal effects.
Though causality-based fair learning is attracting increasing attention, current methods assume the true causal graph is fully known.
This paper proposes a general method to achieve the notion of counterfactual fairness when the true causal graph is unknown. 
To select features that lead to counterfactual fairness, we derive the conditions and algorithms to identify ancestral relations between variables on a \textit{Partially Directed Acyclic Graph (PDAG)}, specifically, a class of causal DAGs that can be learned from observational data combined with domain knowledge. 
Interestingly, we find that counterfactual fairness can be achieved as if the true causal graph were fully known, when specific background knowledge is provided: the sensitive attributes do not have ancestors in the causal graph. Results on both simulated and real-world datasets demonstrate the effectiveness of our method.
\end{abstract}

%% file: Pages/Introduction_v2.tex
\section{Introduction}\label{sec: Introduction}
With the widespread application of machine learning in various fields (e.g., hiring decisions \citep{hoffman2018discretion}, recidivism predictions \citep{dieterich2016compas, brennan2009evaluating},  and finance \citep{sweeney2013discrimination, khandani2010consumer}), the ethical and social impact of machine learning is receiving increasing attention. In particular, machine learning algorithms are sensitive to the bias in training data, which may render their decisions discriminatory against individual or sub-population group with respect to \textit{sensitive attributes}, e.g., gender and race. For example, bias against African-Americans was found with COMPAS, a decision support tool used by U.S. courts to assess the likelihood of a defendant becoming a recidivist \citep{dressel2018accuracy}.  

To achieve fair machine learning, a large body of methods have been proposed to mitigate bias according to different fairness measures. These methods can be roughly categorized into two groups. 
The first group focuses on devising statistical fairness notions, which typically indicate the statistical discrepancy between individuals or sub-populations, e.g., statistical parity~\citep{dwork2012fairness}, equalized odds~\citep{hardt2016equality}, and predictive parity~\citep{chouldechova2017fair}. 
Built on the causal inference framework \citep{pearl2000models}, the second group treats the presence of causal effect of the sensitive attribute on the decision as discrimination~\citep{zhang2017causal,kilbertus2017avoiding, kusner2017counterfactual, zhang2018fairness, zhang2018equality,nabi2018fair,wu2019pc,khademi2019fairness,chiappa2019path, russell2017worlds, zhang2018causal, zhang2017anti, kusner2019making, salimi2019interventional, wu2018discrimination, galhotra2022causal}. 

Among the causal fairness works, counterfactual fairness~\citep{kusner2017counterfactual}, which considers
the causal effect of sensitive attributes on the individual level, has received much attention. Given the true causal graph, Kusner et al. \citep{kusner2017counterfactual} provide the conditions and algorithms for achieving counterfactual fairness in the constructed predictive model. Wu et al. \citep{wu2019pc} and Chiappa \citep{chiappa2019path} extend counterfactual fairness by considering path-specific effects. When the counterfactual effect is unidentifiable, counterfactual fairness can be approximately achieved by lower
and upper bounds of counterfactual effects \citep{wu2019counterfactual}.

Existing counterfactual fairness methods assume the availability of a causal directed acyclic graph (DAG) \citep{pearl2009causality}, which encodes causal relationships between variables. However, in many real-world scenarios, the causal DAG is often unknown due to insufficient understanding of the system under investigation. 
An obvious path forward is to infer the causal DAG from observational data using causal discovery methods \citep{spirtes1991algorithm, colombo2014order, spirtes2000constructing, chickering2002optimal, shimizu2006linear,hoyer2008nonlinear, zhang2009identifiability, peters2014causal, peters2014identifiability}. Unfortunately, without strong assumptions on the data generating process, such as linearity \citep{shimizu2006linear} and additive noise~\citep{hoyer2008nonlinear, peters2014causal}, one cannot uniquely recover the underlying true causal graph from observational data alone. In the general case, causal discovery methods could output a Markov equivalence class of DAGs that encode the same set of conditional independencies from data, which can be represented by a completely partially directed acyclic graph (CPDAG) \citep{spirtes2000causation, chickering2002optimal}. With additional background knowledge, we can discern more causal directions, which can be represented by a maximally partially directed acyclic graph (MPDAG) \citep{meek1995causal}, but we still cannot obtain a unique causal DAG. 

Following this augment, we have a natural question to answer: can we learn counterfactual fairness with a partially known causal graph represented by MPDAG? \footnote{CPDAG is a special case of MPDAG without background knowledge, so we deal with MPDAG generally.}
In a causal DAG, if a variable $S$ has a causal path to another variable $T$, then $S$ is an ancestor of $T$ and $T$ is a descendant of $S$. Counterfactual fairness deems the prediction to be counterfactually fair if it is a function of the non-descendants of sensitive attributes \citep{kusner2017counterfactual}, which are straightforward to identify in DAGs. 
However, in an MPDAG, with respect to a variable $S$, a variable $T$ can be either
\begin{itemize}
\itemsep0em
    \item a \textit{definite descendant} of $S$ if $T$ is a descendant of $S$ in every equivalent DAG,
    \item a \textit{definite non-descendant} of $S$ if $T$ is a non-descendant of $S$ in every equivalent DAG,
    \item a \textit{possible descendant} of $S$ if $T$ is neither a definite descendant nor a definite non-descendant of $S$.
\end{itemize}
To achieve counterfactual fairness in MPDAGs, we need to select the definite non-descendants and some possible descendants of the sensitive attributes that could be non-descendants in the true DAG to make prediction. This comes to the core challenge: the identifiability of the ancestral relation between two distinct variables in an MPDAG. We refer the interested readers to a summary of existing identifiability results in \cref{sec: Related Work}. 

In this paper, we assume no selection bias and presence of confounders because the causal discovery algorithms themselves will not work well in such challenging scenarios. Under this assumption, which can also be found in the most related work \citep{zhang2017causal, chiappa2019path, chikahara2021learning, wu2019counterfactual}, but removing the assumption on a fully directed causal DAG, we make the following main contributions towards achieving counterfactual fairness on MPDAGs:   
\begin{itemize}
\itemsep0em
    \item We provide a sufficient and necessary graphical criterion (\cref{theo}) to check whether a variable is a definite descendant of another variable on an MPDAG;
    \item Based on the proposed criterion, we give an efficient algorithm (\cref{alg: ancestral relation between X and Y}) to identify ancestral relations between any two variables on an MPDAG;
    \item We propose the first approach for achieving counterfactual fairness on partially known causal graphs, specifically MPDAGs;
    \item We find that on an MPDAG, counterfactual fairness can be achieved as if the true causal graph is fully known with the assumption that the sensitive attributes can not have ancestors in the causal graph.
\end{itemize}


%% file: Pages/Background_v2.tex
\section{Background} \label{sec: Preliminary}
In this section, we first review the structural causal model, causal graph and counterfactual inference. Then we introduce counterfactual fairness criterion and its intuitive implications.

\subsection{Structural causal model, causal graph, and counterfactual inference} \label{sec: SCM and CF}
Structural causal model (SCM) \citep{pearl2000models} is a framework to model causal relations between variables. It is defined by a triple $(U,V,F)$, where $V$ are observable endogenous variables and $U$ are unobserved exogenous variables that cannot be caused by any variable in $V$. $F$ is a set of functions ${f_1,...,f_n}$, one for each $V_i \in V$ expressing how $V_i$ is dependant on its direct causes:
$
V_i = f_i(pa_i, U_i),
$ 
where $pa_i$ is the observed direct causes of $V_i$ and $U_i$ is the set of unobserved direct causes of $V_i$. The exogenous $U_i$s are required to be jointly independent. The set of equations $F$ induces a causal graph $\mathcal{D}$ over the variables, usually in the form of a directed acyclic graph (DAG), where the directed causes of $V_i$ represents its parent set in the causal graph.

Based on SCM, one can perform counterfactual inference to answer the problems in the counterfactual world. For example, consider in a fairness context, $A$, $Y$ and $\mathcal{X}$ represent the sensitive attributes, outcome of interest, and other observable attributes, respectively. For an individual $U=u$ with $A=a, Y=y$ and $\mathcal{X}=\mathbf{x}$, a common counterfactual query is: “For this individual $u$, what would the value of $Y$ have been had $A$ taken value $a'$”. The solution, denoted as $Y_{A \leftarrow a'}(u)$ can be obtained by three steps in the deterministic case: Abduction, Action and Prediction \cite[Chapter 7.1]{pearl2000models}.
In probabilistic counterfactual inference, the procedure can be modified to estimate the posterior of $u$ and the distribution of $Y_{A \leftarrow a'}(u)$. 

\textbf{DAGs, PDAGs and CPDAGs.} In a directed acyclic graph (DAG), all edges are directed and there is no directed cycle in the graph; when some edges are undirected, we say it is a partially directed graph (PDAG). A DAG encodes a set of conditional independence relations based on the notion \textit{d-separation} \citep{pearl1988probabilistic}. Multiple DAGs are Markov equivalent if they encode the same set of conditional independence relations. A \textit{Markov equivalence class} of a DAG $\mathcal{D}$ can be uniquely represented by a completed partially directed acyclic graph (CPDAG) $\mathcal{G^*}$, denoted by $[\mathcal{G^*}]$.

\textbf{MPDAGs.} The CPDAGs with background knowledge constraint is known as maximally oriented PDAGs (MPDAGs) \citep{meek1995causal}, which can be obtained by applying Meek's rules R1, R2, R3 and R4 in \citep{meek1995causal}. The Algorithm 1 in \citep{perkovic2017interpreting} can be used to construct the MPDAG $\mathcal{G}$ from the CPDAG $\mathcal{G^*}$ and backgroud knowledge $\mathcal{B}$, where the background knowledge $\mathcal{B}$ is assumed to be the \textit{direct causal information} in the form $X \rightarrow Y$, meaning that $X$ is a direct cause of $Y$. The subset of Markov equivalent DAGs consistent with the background knowledge $\mathcal{B}$ can be uniquely represented by an MPDAG $\mathcal{G}$, denoted by $[\mathcal{G}]$. Both a DAG and a CPDAG can be regarded as special cases of an MPDAG when the background knowledge is completely known and not known, respectively.


\subsection{Counterfactual fairness} \label{sec: Counterfactual fairness}
Counterfactual fairness \citep{kusner2017counterfactual} is a fairness criterion based on SCM \citep{pearl2000models}. 
Let $A$, $Y$ and $\mathcal{X}$ represent the sensitive attributes, outcome of interest and other observable attributes, respectively and the prediction of $Y$ is denoted by $\hat{Y}$. For an individual with $\mathcal{X}=\mathbf{x}$ and $A=a$, we say the prediction $\hat{Y}$ is counterfactually fair if it would have been the same had $A$ been $a'$ in the counterfactual world as in the real world that $A$ is $a$.


\begin{definition}[Counterfactual fairness] \citep[Definition 5]{kusner2017counterfactual} \label{def: counterfactual fairness}
    We say the prediction $\hat{Y}$ is counterfactually fair if under any context $\mathcal{X} =\mathbf{x}$ and $A=a$,
    \begin{align}\label{eq: counterfactual fairness}
        &P(\hat{Y}_{A \gets a}(U)=y|\mathcal{X}=\mathbf{x},A=a)
        =P(\hat{Y}_{A\gets{a'}}(U)=y|\mathcal{X}=\mathbf{x},A=a),    \nonumber
    \end{align}
    for all $y$ and any value $a^{\prime}$ attainable by A.
\end{definition}


The definition of counterfactual fairness immediately suggests the following approach in \cref{lem: counterfactual fairness implication} to design a counterfactually fair model. 

\begin{lemma}\citep[Lemma 1]{kusner2017counterfactual}\label{lem: counterfactual fairness implication}
    Let $\mathcal{G}$ be the causal graph of the given model $(U,V,F)$. Then $\hat{Y}$ will be counterfactually fair if it is a function of the non-descendants of $A$.
\end{lemma}



%% file: Pages/Problem.tex
\section{Problem formulation}\label{sec: Problem Formulation}
In this section, we introduce the task of achieving counterfactual fairness given PDAGs, especially MPDAGs that can be learned from observational data using causal discovery algorithms \citep{spirtes1991algorithm,colombo2014order,chickering2002learning}. 

\cref{lem: counterfactual fairness implication} in \cref{sec: Counterfactual fairness} implies learning a counterfactually fair prediction can be framed as selecting the non-descendants of $A$ to predict $Y$. If a causal DAG is used to encode the causal relations of all attributes, finding all non-descendants of $A$ is straightforward. However, as mentioned in \cref{sec: Introduction}, given observational data and optional background knowledge about direct causal information, we can only learn a CPDAG or an MPDAG $\mathcal{G}$, instead of the true DAG $\mathcal{D} \in [\mathcal{G}]$. Unfortunately, not all ancestral relations between $A$ and attributes in $\mathcal{X}$ are identifiable in a CPDAG or MPDAG.

Therefore, to achieve counterfactually fair prediction, we have two problems to solve:
\begin{itemize}
\itemsep0em
    \item Identify the type of ancestral relations of any other attributes with $A$ in $\mathcal{G}$, {\it i.e.}, identifying the definite non-descendants, definite descendants, and possible descendants of $A$ in an MPDAG (Section \ref{sec: Identifiability of Ancestral Relations in MPDAGs});
    \item Build a counterfactually fair model based on the identified ancestral relations (Section \ref{sec: Counterfactual Fairness in MPDAGs}).
\end{itemize}



%% file: Pages/Method.tex
\section{Identifiability of ancestral relations in MPDAGs}\label{sec: Identifiability of Ancestral Relations in MPDAGs}

In this section, we give a sufficient and necessary graphical criterion on identifying the definite ancestral relations between two distinct vertices in an MPDAG. We also provide an efficient algorithm for implementing the proposed criterion. We denote the parents, children, siblings and adjacencies of the node $W$ in a graph $\mathcal{G}$ as $pa(W, \mathcal{G})$, $ch(W, \mathcal{G})$, $sib(W, \mathcal{G})$ and $adj(W, \mathcal{G})$, respectively. A \textit{chord} of a path in $\mathcal{G}$ is any edge joining two non-consecutive vertices on the path. A path without any chord is called \textit{chordless path}. In a graph $\mathcal{G}=(V,E)$, where $V$ and $E$ represent the node set and edge set in $\mathcal{G}$, the \textit{induced subgraph} of $\mathcal{G}$ over $V' \subset V$ is the graph with vertex $V'$ and edges between vertices in $V'$, that is $E' \subset E$. A graph is \textit{complete} if any two distinct vertices are adjacent. 

\subsection{Graphical criterion on identifying ancestral relations in MPDAGs}\label{sec: Graphical criterion}

We first introduce the term \textit{b-possibly causal path} \citep{perkovic2017interpreting} in MPDAGs where the prefix $b$- stands for background.

\begin{definition} [b-possibly causal path, b-non-causal path] \citep[Definition 3.1]{perkovic2017interpreting}\label{b-path}
    Suppose $p=\langle S=V_{0},..., V_{k}=T\rangle $ is a path from $S$ to $T$ in an MPDAG $\mathcal{G}$, $p$ is b-possibly causal in $\mathcal{G}$ if and only if no edge $V_i \leftarrow V_j$, $0 \leq i \le j \leq k $ is in $\mathcal{G}$, including the edge not on $p$. Otherwise, $p$ is a b-non-causal path in $\mathcal{G}$.
\end{definition}

Perkovi{\'c} et al. \citep{perkovic2017interpreting} state that $T$ is a definite non-descendant of $S$ in an MPDAG $\mathcal{G}$ if and only if there is no b-possibly causal path from $S$ to $T$ in $\mathcal{G}$. In this section, to identify whether $T$ is a definite descendant of $S$ in an MPDAG, we provide a sufficient and necessary condition. 

We introduce the term \textit{critical set} \citep{fang2020ida} in CPDAGs as follows. 
\begin{definition} [Critical Set] \citep[Definition 2]{fang2020ida} \label{def: Critical Set}
    Let $\mathcal G$ be an CPDAG, $S$ and $T$ be two distinct vertices in $\mathcal G$. The critical set of $S$ with respect to $T$ in $\mathcal G$ consists of all adjacent vertices of $S$ lying on at least one chordless possibly causal path from $S$ to $T$.
\end{definition}
\cref{def: Critical Set} can be extended to MPDAGs directly,
based on which, we provide a sufficient and necessary graphical condition in \cref{lem: definite ancestral relations in Markov equivalent DAGs} for identifying the definite ancestral relation between two distinct vertices in an MPDAG. 

\begin{lemma}\label{lem: definite ancestral relations in Markov equivalent DAGs}
    Let $\mathcal G$ be an MPDAG and $S$, $T$ be two distinct vertices in $\mathcal G$, then $T$ is a definite descendant of $S$ in $\mathcal G$ if and only if the critical set of $S$ with respect to $T$ always contains a child of $S$ in every DAG $\mathcal D \in [\mathcal G]$.
\end{lemma}

The proof of \cref{lem: definite ancestral relations in Markov equivalent DAGs} is in \cref{proof_lem: definite ancestral relations in Markov equivalent DAGs}. Note that we have to enumerate all Markov equivalent DAGs for checking the condition given in \cref{lem: definite ancestral relations in Markov equivalent DAGs}. To resolve this problem, we provide a condition in \cref{lem: graphical characteristic of critical set} to check the graphical characteristic of the corresponding critical set in the MPDAG directly. The graphical criterion in \cref{lem: graphical characteristic of critical set} has been proved by \cite[Lemma 2]{fang2022local} and \cite[Lemma 1]{chen2021definite} for CPDAGs. We extend it to general MPDAGs here. 


\begin{lemma}\label{lem: graphical characteristic of critical set}
    Let $\mathcal{G}$ be an MPDAG and $S$, $T$ be two distinct vertices in $\mathcal{G}$. Denote by $\mathbf{C}$ the critical set of $S$ with respect to $T$ in $\mathcal{G}$, then $\mathbf{C} \cap{ch(S, \mathcal D)} = \emptyset$ for some DAG $\mathcal{D} \in [\mathcal{G}]$, if and only if $\mathbf{C} = \emptyset$, or $\mathbf{C}$ induces a complete subgraph of $\mathcal{G}$ but $\mathbf{C} \cap{ch(S,\mathcal{G})}=\emptyset$.
\end{lemma}

 The proof of \cref{lem: graphical characteristic of critical set} is in \cref{proof_lem: graphical characteristic of critical set}. 
Building on \cref{lem: definite ancestral relations in Markov equivalent DAGs} and \cref{lem: graphical characteristic of critical set}, we arrived at the desired sufficient and necessary graphical criterion in \cref{theo}. 

\begin{theorem}\label{theo}
    Let $S$ and $T$ be two distinct vertices in an MPDAG $\mathcal G$, and $\mathbf{C}$ be the critical set of $S$ with respect to $T$ in $\mathcal G$. Then $T$ is a definite descendant of $S$ if and only if either $S$ has a definite arrow into $\mathbf{C}$, that is $\mathbf{C} \cap{ch(S, \mathcal G)} \ne \emptyset$, or $S$ does not have a definite arrow into $\mathbf{C}$ but $\mathbf{C}$ is non-empty and induces an incomplete subgraph of $\mathcal{G}$.
\end{theorem}
 With \cref{theo}, we can identify whether $T$ is a definite descendant of $S$ in an MPDAG by finding all chordless b-possibly causal path from $S$ to $T$ and checking graphical characteristic of the corresponding critical set. We provide an example to illustrate \cref{theo} in \cref{Appendix: An illustration Example}.

\subsection{Algorithms} \label{sec: Algorithms}
Here we provide an efficient algorithm to identify the ancestral relation between any two distinct vertices in an MPDAG based on the theoretical results in \cref{theo}. It is straightforward to identify the non-ancestral relation by checking if there exists a b-possibly causal path from the source to the target. However, discriminating between a definite ancestral relation and a possible ancestral relation in an MPDAG by \cref{theo} requires finding the critical set. 

According to the definition of the critical set, we need to find all chordless possibly causal paths from the source variable to the target variable first. However, it is cumbersome to check whether a path is chordless, since it involves considering many edges not on the path. In \cref{prop: obtain critical set by ds path} we propose a more efficient way to find critical sets by leveraging the relation between chordless path and definite status path in \cref{lem: ds and chordless}. We first provide the notion of \textit{collider} and \textit{definite status path}.

\textbf{Colliders and Definite Status paths.}
In a path $p=\langle S=V_{0},..., V_{k}=T\rangle $, if $V_{i-1} \rightarrow V_{i}$ and $V_{i} \leftarrow V_{i+1}$, then we say $V_{i}$ is a \textit{collider} on $p$. If $V_{i-1}$ and $V_{i+1}$ are not adjacent, then the triple $\langle V_{i-1}, V_{i}, V_{i+1}\rangle $ is called a \textit{v-structure collided} on $V_{i}$, and $V_{i}$ can be called \textit{unshielded collider}. A node $V_{i}$ is a \textit{definite non-collider} on a path $p$ if there is at least one edge out of $V_{i}$ on $p$ or $V_{i-1} - V_{i} - V_{i+1}$ is a subpath of $p$ and $V_{i-1}$ is not adjacent with $V_{i+1}$. A node is of \textit{definite status} on a path if it is a collider, a definite non-collider or an endpoint on the path. A path $p$ is of definite status if every node on $p$ is of definite status.


\begin{lemma}\label{lem: ds and chordless}
    Let $S$ and $T$ be two distinct vertices in an MPDAG $\mathcal{G}$. If $p$ is a chordless path from $S$ to $T$ in $\mathcal{G}$, then $p$ is of definite status.
\end{lemma}


\begin{proposition}\label{prop: obtain critical set by ds path}
    Let $\mathcal{G}$ be an MPDAG, $S$ and $T$ be two distinct vertices in $\mathcal{G}$. Denote by $\mathbf{C}_{ST}$ the critical set of $S$ with respect to $T$ in $\mathcal{G}$, then $\mathbf{C}_{ST}=\mathbf{F}_{ST}$, where $\mathbf{F}_{ST}$ denotes all adjacent vertices of $S$ lying on at least one b-possibly causal path of definite status from $S$ to $T$ in $\mathcal{G}$ that there is no chord with $S$ as an endpoint.
\end{proposition}
\cref{prop: obtain critical set by ds path} implies that, instead of enumerating all chordless possibly causal path from the source vertex to the target vertex, we enumerate b-possibly causal paths of definite status, excluding paths that contain any chord where the source vertex is an end node. The nodes adjacent to the source vertex on these paths constitute the desired critical set.

The proof of \cref{lem: ds and chordless} and \cref{prop: obtain critical set by ds path} are in \cref{proof_lem: ds and chordless} and \cref{proof_prop: obtain critical set by ds path}, respectively.

Following from \cref{prop: obtain critical set by ds path}, we develop \cref{alg: FindingCriticalSet} showing how to efficiently find the critical set of $S$ w.r.t. $T$ in an MPDAG $\mathcal{G}$. 
Given an MPDAG $\mathcal{G}$, the source vertex $S$ and the target vertex $T$, the output of \cref{alg: FindingCriticalSet} is the critical set of $S$ w.r.t. $T$ in $\mathcal{G}$. Using the breadth-first-search algorithm, \cref{alg: FindingCriticalSet} searches b-possibly causal path of definite status without any chord ending in $S$ from $S$ to $T$. A detailed explanation of this algorithm is provided in \cref{Appendix: Detailed Explaination of algo}. 

\begin{algorithm}
\caption{Finding the critical set of $S$ with respect to $T$ in an MPDAG}
\label{alg: FindingCriticalSet}
\begin{algorithmic}[1]
\State \textbf{Input:} MPDAG $\mathcal{G}$, two distinct vertices $S$ and $T$ in $\mathcal{G}$.
\State \textbf{Output:} The critical set $\mathbf{C}$ of $S$ with respect to $T$ in $\mathcal{G}$.
\State Initialize $\mathbf{C}=\emptyset$, a waiting queue $\mathcal{Q}=[]$, and a set $\mathcal{H}=\emptyset$,
\For{$\alpha \in sib(S) \cup ch(S)$}
    \State add $(\alpha, S, \alpha)$ to the end of $\mathcal{Q}$,
\EndFor
\While{$\mathcal{Q} \neq \emptyset$}
    \State take the first element $(\alpha, \phi, \tau)$ out of $\mathcal{Q}$ and add it to $\mathcal{H}$;
    \If{$\tau = T$}
        \State {add $\alpha$ to $\mathbf{C}$, and remove from $\mathcal{S}$ all triples where the first element is $\alpha$;}
    \Else
        \For{each node $\beta$ in $\mathcal{G}$}
            \If{$\tau \rightarrow \beta$ or $\tau - \beta$}
                \If{$\tau \rightarrow \beta$ or $\phi$ is not adjacent with $\beta$ or $\tau$ is the endnode}
                    \If{$\beta$ and $S$ are not adjacent}
                        \If{$(\alpha, \tau, \beta) \notin \mathcal{H}$ and $(\alpha, \tau, \beta) \notin \mathcal{Q}$}
                        \State add $(\alpha, \tau, \beta)$ to the end of $\mathcal{Q}$,
                        \EndIf
                    \EndIf
                \EndIf
            \EndIf
        \EndFor
    \EndIf
\EndWhile
\State \textbf{return} $\mathbf{C}$
\end{algorithmic}
\end{algorithm}

Finally, we present \cref{alg: ancestral relation between X and Y} to identify the type of ancestral relation. We first find the critical set $\mathbf{C}$ of $S$ with respect to $T$. When $\mathbf{C} = \emptyset$, there is no b-possibly causal path from $S$ to $T$, so $T$ is a definite non-descendant of $S$. When $\mathbf{C} \neq \emptyset$, using \cref{theo}, we can decide whether $T$ is a definite descendant or a possible descendant of $S$. 

\begin{algorithm}
\begin{algorithmic}[1]
\caption{Identify the type of ancestral relation of $S$ with respect to $T$ in an MPDAG}
\label{alg: ancestral relation between X and Y}
\State \textbf{Input:} MPDAG $\mathcal{G}$, two distinct variables $S$ and $T$ in $\mathcal{G}$.
\State \textbf{Output:} The type of ancestral relation between $S$ and $T$.
\State Find the critical set $\mathbf{C}$ of $S$ with respect to $T$ in $\mathcal{G}$ by Algorithm \ref{alg: FindingCriticalSet}.
\If{$|\mathbf{C}|=0$}
    \State \textbf{return} $T$ is a definite non-descendant of $S$.
\EndIf
\If{$S$ has an arrow into $\mathbf{C}$ or $\mathbf{C}$ induces an incomplete subgraph of $\mathcal{G}$}
    \State \textbf{return} $T$ is a definite descendant of $S$.
\EndIf
\State \textbf{return} $T$ is a possible descendant of $S$.
\end{algorithmic}
\end{algorithm}

Since in \cref{alg: FindingCriticalSet}, the same triple like $(\alpha, \phi, \tau)$ can only be visited at most once, where $\alpha$ is a sibling or a child of $S$ in the MPDAG $\mathcal{G}$, $\tau$ is a node on a b-possibly causal path of definite status from $S$ to $T$ without any chord ending in $S$, and $\phi$ lies immediately before $\tau$ on such path. The complexity of \cref{alg: FindingCriticalSet} in the worst case is $\mathcal{O}(|sib(S,\mathcal{G})+ch(S,\mathcal{G})|*|E(\mathcal{G})|)$, where $|E(\mathcal{G})|$ is the number of edges in $\mathcal{G}$. Consequently, the computational complexity of \cref{alg: ancestral relation between X and Y} is $\mathcal{O}(|sib(S,\mathcal{G})+ch(S,\mathcal{G})|*|E(\mathcal{G})|)$.

\section{Counterfactual fairness in MPDAGs}\label{sec: Counterfactual Fairness in MPDAGs}
Now, we return to our problem of learning counterfactually fair models via selecting features from $\mathcal{X}$. We will consider two cases: 1) a general MPDAG and 2) an MPDAG learned with background knowledge that $A$ is a root node.


\subsection{General case} \label{sec: General Case}
We can identify the set of definite descendants, possible descendants and definite non-descendants of the sensitive attribute by simply applying \cref{alg: ancestral relation between X and Y} to each pair of sensitive and any other attribute, which is concluded in \cref{alg: ancestral relation between X and all the other nodes}.  The detailed description of \cref{alg: ancestral relation between X and all the other nodes} is provided in \cref{Appendix: Additional Algorithms}. As the computational complexity of \cref{alg: ancestral relation between X and Y} in the worst case is $\mathcal{O}(|sib(S,\mathcal{G})+ch(S,\mathcal{G})|*|E(\mathcal{G})|)$, the complexity of \cref{alg: ancestral relation between X and all the other nodes} is directly $\mathcal{O}(|sib(S,\mathcal{G})+ch(S,\mathcal{G})|*|E(\mathcal{G})|*|V(\mathcal{G})|)$, where $|E(\mathcal{G})|$ is the number of edges and $|V(\mathcal{G})|$ is the number of nodes in $\mathcal{G}$. We consider two feature selection methods. The first method (called \texttt{Fair}) only selects the definite non-descendants, which ensures counterfactual fairness. However, the number of definite non-descendants in an MPDAG might be too small, resulting in low prediction accuracy. Therefore, we also propose a second method (called \texttt{FairRelax}),which uses possible descendants of $A$ to increase the prediction accuracy at the cost of a violation of counterfactual fairness.

\subsection{Under root node assumption} \label{sec: Counterfactual Fairness in MPDAGs/Under Root Node Assumption}
Kusner et al. \cite[Section 3.2]{kusner2017counterfactual} mentioned the ancestral closure of sensitive attributes, meaning that typically we should expect the sensitive attribute set $A$ to be closed under ancestral relationships given by the causal graph. For instance, in the example that religion can be affected by the geographical place of origin, if \textit{religion} is a sensitive attribute and \textit{geographical place of origin} is a parent of \textit{religion}, then it should also be in $A$. Therefore, $A$ will be the root node in most cases except some counterintuitive scenarios. Thus, we consider the following assumption, which is often true in real-world datasets. 
\begin{assumption} \label{assum: A is a root node}
    The sensitive attribute can only be a root node 
    in a causal MPDAG.
\end{assumption}
For example, `sex' cannot be caused by factors like `education' and `salary'. 
In an MPDAG $\mathcal{G}$, given \cref{assum: A is a root node}, the ancestral relation between the sensitive attribute and any other attribute is fully identified, as shown in the following proposition.
\begin{proposition} \label{prop: no possible ancestral relations in CF in MPDAGs}
In a MPDAG $\mathcal{G}$ with sensitive attribute $A$, if \cref{assum: A is a root node} holds, then any other attribute is either a definite descendant or definite non-descendant of $A$. Moreover, an arbitrary attribute $W$ is a definite descendant of $A$ if and only if there is a causal path from $A$ to $W$ in $\mathcal{G}$. 
\end{proposition}
The proof of \cref{prop: no possible ancestral relations in CF in MPDAGs} is in \cref{proof_prop: no possible ancestral relations in CF in MPDAGs}. From \cref{prop: no possible ancestral relations in CF in MPDAGs}, 
it is very interesting to see that fitting a model with the definite non-descendants of $A$ in $\mathcal{G}$ is exactly the same thing as fitting a model with the non-descendants of $A$ in the true DAG $\mathcal{D}$. Thus, the counterfactual fairness can be achieved as if the true causal DAG is fully known. 

Under \cref{assum: A is a root node}, since there is a causal path from $A$ to any definite descendant of $A$, we can directly identify whether a target attribute is a descendants of $A$ by checking if there is a causal path from $A$ to the target. To find all definite descendants in an MPDAG, we can use a breath first search algorithm with computational complexity $\mathcal{O}(|V|+|E|)$, where $|V|$ is the number of nodes and $|E|$ is the number of edges in $\mathcal{G}$. Then the remaining nodes are definite non-descendants of $A$ in $\mathcal{G}$.




%% file: Pages/Experiment.tex
\section{Experiment}\label{sec: Experiment}

In this section, we illustrate our approach on a simulated and a real-world dataset by evaluating the prediction performance and fairness of our approach. The prediction performance is evaluated by root mean squared error (RMSE). The counterfactual unfairness can be measured by the discrepancy of the predictions in the real world and counterfactual world for each individual. In addition, unfairness can also be revealed by comparing the distributions of the predictions in two worlds, which coincide if the prediction is counterfactually fair. 

\textbf{Baselines.} We consider three baselines: 1) \texttt{Full} is a standard model that uses all attributes, including the sensitive attributes to make predictions, 2) \texttt{Unaware} is a model that uses all attributes except the sensitive attributes to make predictions, and 3) \texttt{Oracle} is a model that makes predictions with all attributes that are non-descendants of the sensitive attribute given the groundtruth DAG. As mentioned in Section \ref{sec: Counterfactual Fairness in MPDAGs}, 
our proposed methods include \texttt{FairRelax}, which makes predictions using all definite non-descendants and possible descendants of the sensitive attribute in an MPDAG, and \texttt{Fair}, which makes predictions using all definite non-descendants of the sensitive attribute in an MPDAG.


\subsection{Synthetic data} \label{sec: Experiment/Synthetic Data}
The synthetic data is generated from linear structural equation models according to a random DAG. As the simulated DAG is known, we obtain the CPDAG from the true DAG without running the causal discovery algorithms.\footnote{Given a sufficiently large sample size, current causal discovery algorithms can recover the CPDAG with high accuracy on the simulated data \citep{glymour2019review}.} We also add background knowledge to turn a CPDAG into its corresponding MDPAG. 

We first randomly generate DAGs with $d$ nodes and $2d$ directed edges from the graphical model Erd{\H{o}}s-R{\'e}nyi (ER), where $d$ is chosen from $\{10,20,30,40\}$. For each setting, we generate $100$ DAGs. For each DAG $\mathcal{D}$, two nodes are randomly chosen as the outcome and the sensitive attribute, respectively. The sensitive attribute can have two or three values, drawn from a Binomial([0,1]) or Multinomial([0,1,2]) distribution separately. The weight, $\beta_{ij}$, of each directed edges $X_i \rightarrow X_j$ in the generated DAG, is drawn from a Uniform($[-2,-0.5]\cup [0.5,2]$) distribution. The data are generated according to the following linear structural equation model:
\begin{equation} \label{eq: linear structual equation model}
    X_i = \sum_{X_j \in pa(X_i)} \beta_{ij}X_j + \epsilon_i, i=1,...,n,
\end{equation}
where $\epsilon_1,...,\epsilon_n$ are independent $N(0,1.5)$. Then we generate one sample with size $1000$ for each DAG. The proportion of training and test data is splitted as $0.8:0.2$. Once the CPDAG $\mathcal{G^*}$ is obtained, where $\mathcal{D} \in [\mathcal{G^*}]$, we randomly generate the direct causal information $A \rightarrow B$ as the background knowledge from the edges where $A \rightarrow B$ is in DAG $\mathcal{D}$, while $A - B$ is in CPDAG $\mathcal{G^*}$. We show a randomly generated DAG $\mathcal{D}$, the corresponding CPDAG $\mathcal{G^*}$ and MPDAG $\mathcal{G}$ as in \cref{fig: SimulationGraph}, see \cref{Appendix: A randomly generated causal graph example}. For additional experiments based on 
more complicated structural equations and varying amount of possible background knowledge, please refer to \cref{Appendix: Experiment based on non-linear structural equations} and \cref{Appendix: Experiment analyzing fairness performance with varying amount of given domain knowledge}. We also analyze the model robustness experimentally on causal discovery algorithms in \cref{Appendix: Experiment analyzing method robustness on causal discovery algorithm}.

\textbf{Counterfactual fairness.} According to the predefined linear causal model, we first generate the counterfactual data given counterfactual sensitive attributes. For each individual, the noise of any counterfactual feature is the same as that in the observational data. In order to evaluate the counterfactual fairness of the baseline methods, we sample data from both the original and counterfactual data and fit them with all the models. 
Here, unfairness can be measured by the absolute difference of two predictions, $\hat{Y}_{A \leftarrow a}(u)$ and $\hat{Y}_{A \leftarrow a'}(u)$. The results for each model with different graph settings are shown in \cref{tab: RMSE and Unfairness for simultion data} and \cref{fig: Boxplot_unfairness}. Obviously, \texttt{Oracle} and \texttt{Fair} is counterfactually fair, since they do not use any feature that is causally dependant on the sensitive attribute. \texttt{Full} and \texttt{Unaware} have high counterfactual unfairness, while our \texttt{FairRelax} has very low counterfactual unfairness. Additionally, when the model is counterfactually fair, the distributions of the predictions in two worlds should lie on top of each other, as in the \texttt{Oracle} and \texttt{Fair} models. Although counterfactual unfairness is exhibited in the other three models, \texttt{FairRelax} is closer to strictly counterfactually fair methods. An exemplary density plot of the predictions of all the models in one original and counterfactual dataset is 
shown in \cref{fig: Densityplot_unfairness_simulation}.

\begin{table}[!ht]
\caption{Average unfairness and RMSE for synthetic datasets on held-out test set. For each graph setting, the unfairness gets decreasing from left to right and the RMSE gets increasing from left to right.}
\label{tab: RMSE and Unfairness for simultion data}
\begin{center}
\small
\begin{tabular}{cccccccc}
    \hline
           &Node &Edge   &Full     &Unaware  &FairRelax  &Oracle   &Fair\\ \hline
  \multirow{4}{*}{\begin{turn}{90}Unfairness\end{turn}}  &$10$   &$20$    &$0.288\pm0.363$ &$0.200\pm0.322$  &$0.023\pm0.123$ &$0.000\pm0.000$ &$0.000\pm0.000$ \\ 
            &$20$   &$40$    &$0.203\pm0.341$ &$0.165\pm0.312$ &$0.019\pm0.145$ &$0.000\pm0.000$ &$0.000\pm0.000$ \\ 
            &$30$   &$60$    &$0.155\pm0.304$ &$0.143\pm0.312$ &$0.020\pm0.123$ &$0.000\pm0.000$ &$0.000\pm0.000$ \\ 
            &$40$   &$80$    &$0.095\pm0.189$ &$0.075\pm0.182$ &$0.009\pm0.055$ &$0.000\pm0.000$ &$0.000\pm0.000$ \\ \hline\hline
    
    \multirow{4}{*}{\begin{turn}{90}RMSE\end{turn}}    &$10$   &$20$    &$0.621\pm0.251$   &$0.637\pm0.261$  &$1.031\pm0.751$   &$1.065\pm0.751$    &$1.137\pm0.824$ \\ 
            &$20$   &$40$    &$0.595\pm0.255$   &$0.599\pm0.253$   &$0.818\pm0.488$   &$0.847\pm0.55$    &$0.952\pm0.645$ \\ 
            &$30$   &$60$    &$0.597\pm0.24$    &$0.601\pm0.242$  &$0.797\pm0.489$    &$0.849\pm0.644$   &$1.024\pm0.908$\\ 
            &$40$   &$80$    &$0.600\pm0.273$     &$0.601\pm0.272$ &$0.755\pm0.441$    &$0.766\pm0.452$ &$0.800\pm0.480$ \\ \hline
\end{tabular}
\end{center}
\end{table}

\begin{figure}[ht]
\label{fig: Boxplot_unfairness_and_RMSE}
\centering
\subfloat[Average unfairness for each model and graph setting.]{
\label{fig: Boxplot_unfairness}
\includegraphics[width=0.48\columnwidth,height=0.3\columnwidth]{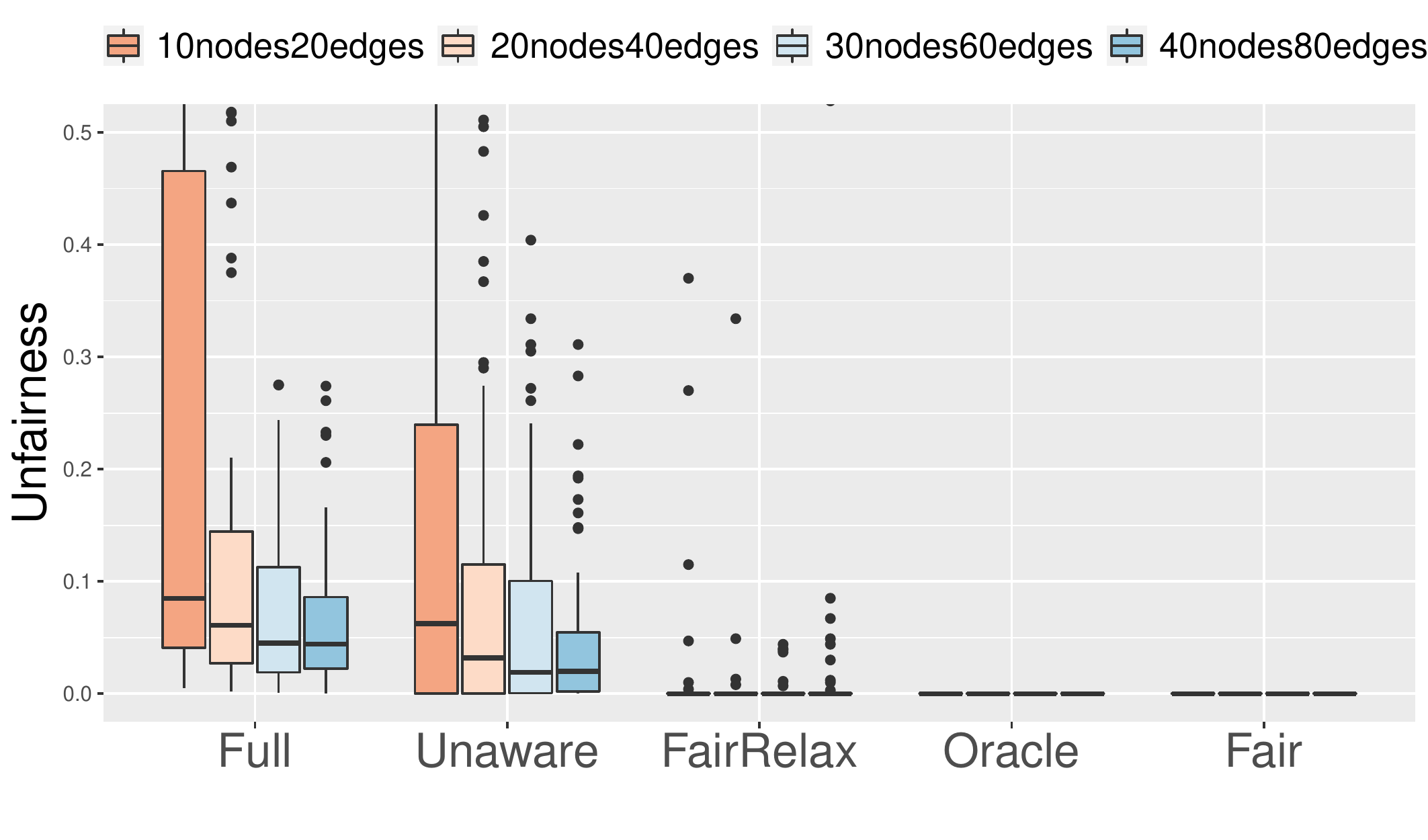}
}
\subfloat[Average RMSE for each model and graph setting.]{
\label{fig: Boxplot_RMSE}
\includegraphics[width=0.48\columnwidth,height=0.3\columnwidth]{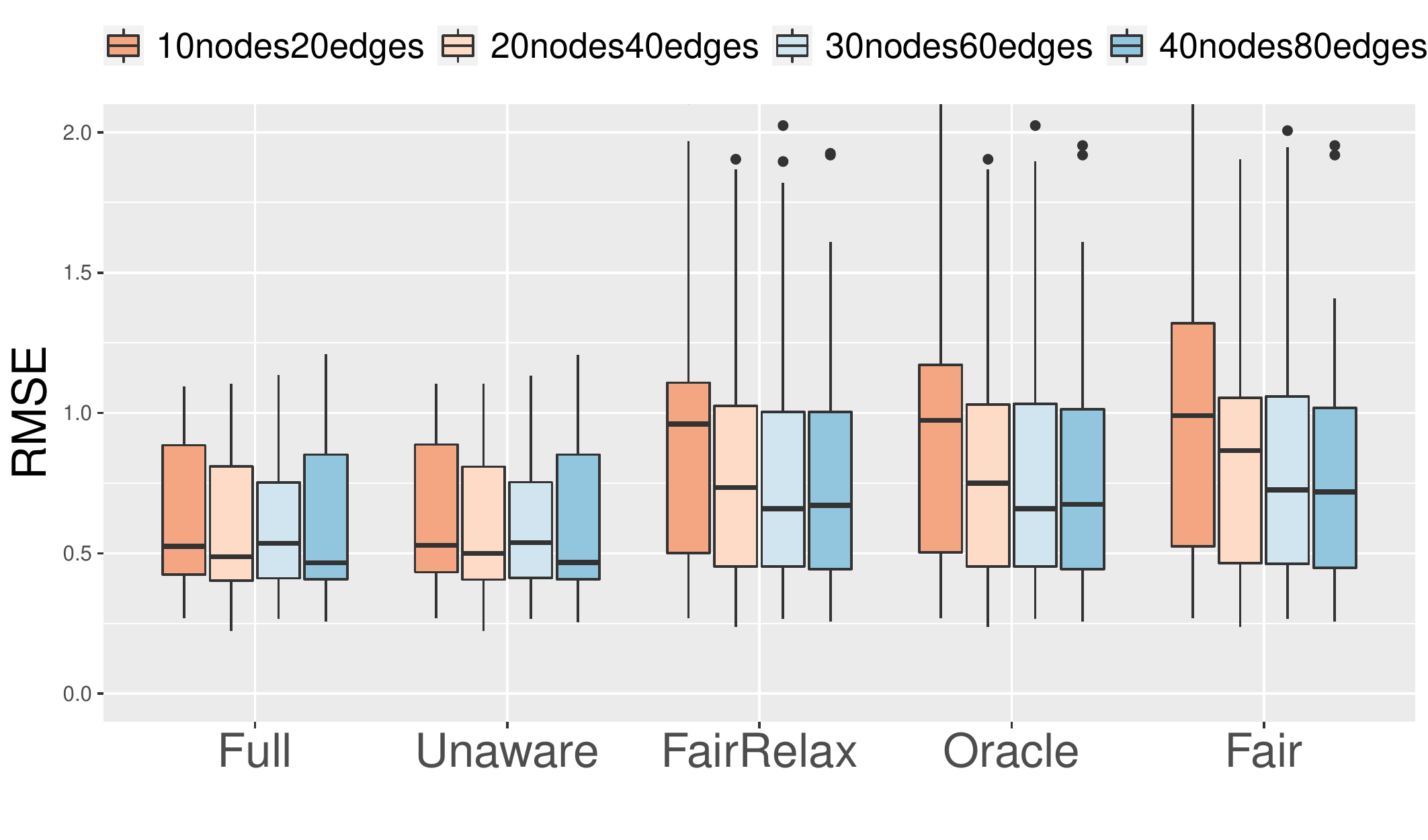}
}
\caption[]{Average unfairness and RMSE for synthetic datasets on held-out test set. For each graph setting, the unfairness gets decreasing from left to right, while RMSE has the opposite trend. The extreme unfairness and high RMSE is because almost all attributes are descendants of the sensitive attribute in around $10/100$ randomly generated graphs. This also explains why the standard deviation of unfairness for the model \texttt{Full}, \texttt{Unware}, \texttt{FairRelax} and RMSE for \texttt{FairRelax}, \texttt{Oracle} and \texttt{Fair} is that large in \cref{tab: RMSE and Unfairness for simultion data}.}
\vskip 0.2in
\end{figure}

\begin{figure}[htp]
\begin{center}
\centerline{\includegraphics[width=0.96\columnwidth,height=0.2\columnwidth]{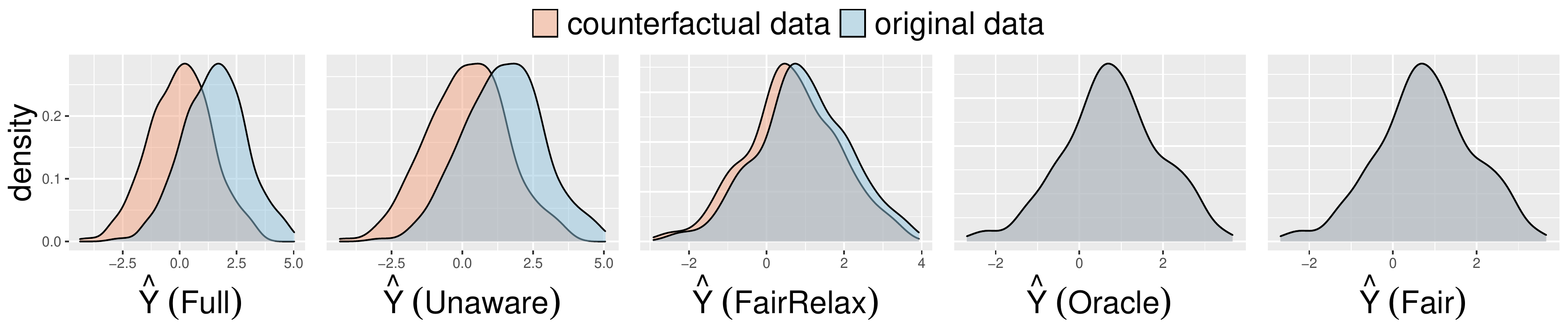}}
\caption{Density plot of the predicted $Y_{A \leftarrow a}(u)$ and $Y_{A \leftarrow a'}(u)$ in synthetic data.}
\label{fig: Densityplot_unfairness_simulation}
\end{center}
\vskip -0.2in
\end{figure}

\textbf{Accuracy.} For each graph setting, we report average RMSE achieved on $100$ causal graphs by fitting a linear regression model for the baselines and our proposed models in \cref{tab: RMSE and Unfairness for simultion data} and \cref{fig: Boxplot_RMSE}. We can observe that, for each graph setting, the \texttt{Full} model obtains the lowest RMSE, which is not surprising as it uses all features. In addition, our \texttt{FairRelax} methods obtains better accuracy than the strictly counterfactually fair methods \texttt{Fair} and \texttt{Oracle}, which is consistent with the accuracy-fairness trade-off phenomenon. More discussion on accuracy-fairness trade-off is included in \cref{Appendix: Discussion on Accuracy-fairness trade-off}.


\subsection{Real data}
The UCI Student Performance Data Set \citep{cortez2008using} regarding students performance in Mathematics is used in this experiment. The data attributes include student grade in secondary education, demographic, social, and school related features. It contains 395 students records with $32$ attributes. We regard \textit{sex} as the sensitive attribute in this dataset. Besides, we remove the first, second and final period grade, denoted by \textit{G1,G2}, and \textit{G3} in the dataset and generate the value of target attribute \textit{Grade} as the average of \textit{G1,G2}, and \textit{G3}.

In this section, we first learn the corresponding CPDAG from this dataset leveraging the GES structure learning algorithm \citep{chickering2002learning}, which is implemented by a general causal discovery software - TETRAD \citep{ramsey2018tetrad}. After uploading the preprocessed data, we can learn the CPDAG $\mathcal{G^*}$. The evolution of the CPDAG to MPDAG is shown in \cref{Appendix: Causal graphs for Student Dataset}. 

Our experiments are carried out under the root node assumption on the MPDAG $\mathcal{G}$ in \cref{fig: real_data_with_assumption}, as \textit{sex} cannot be caused by other variables in the dataset. Due to the space limit, \cref{fig: real_data_with_assumption} is provided in \cref{Appendix: Causal graphs for Student Dataset}. Thus, the ancestral relations can be fully identified according to our theoretical results in Section \ref{sec: Counterfactual Fairness in MPDAGs/Under Root Node Assumption}. Our algorithm will find all the non-descendants of the sensitive attribute. On this dataset, the definite descendants of \textit{sex} can be identified as \textit{\{Walc, goout, Dalc, studytime\}} and all the other nodes are definite non-descendants of \textit{sex}.

The counterfactual fairness and accuracy are measured in almost the same way as in \cref{sec: Experiment/Synthetic Data}. The details on counterfactual data generation and model fitting can be referred to \cref{Appendix: Training Details on Real Data}. The results are reported in \cref{tab: RMSE and Unfairness for real data}. Since under the root node assumption, there is no possible descendants of the sensitive attribute, the model \texttt{Fair} and \texttt{FairRelax} give the same RMSE result and both of them achieve counterfactual fairness at the cost of slight accuracy decrease. Instead, the model \texttt{Full} and \texttt{Unaware} are unfair and the \texttt{Full} model is more unfair than \texttt{Unaware}. Besides, the distribution of predictions for the original and counterfactual data for all models have the same trend as the synthetic data. \cref{fig: Densityplot_unfairness_real_data} is the corresponding density plot.

\begin{table}[t]
\caption{Average RMSE and unfairness for Student dataset on held-out test sets.}
\label{tab: RMSE and Unfairness for real data}
\begin{center}
\small
\begin{tabular}{cccccc}
\hline
        & Full & Unaware & FairRelax  & Fair\\ \hline
Unfairness &$0.761\pm0.228$ &$0.250\pm0.085$ &$0.000\pm0.000$ &$0.000\pm0.000$ \\\hline
RMSE    & $3.49\pm0.292$ & $3.482\pm0.297$ & $3.509\pm0.356$ &$3.509\pm0.356$ \\\hline
\end{tabular}
\end{center}
\end{table}

\begin{figure}[htp]
\begin{center}
\centerline{\includegraphics[width=0.98\columnwidth,height=0.2\columnwidth]{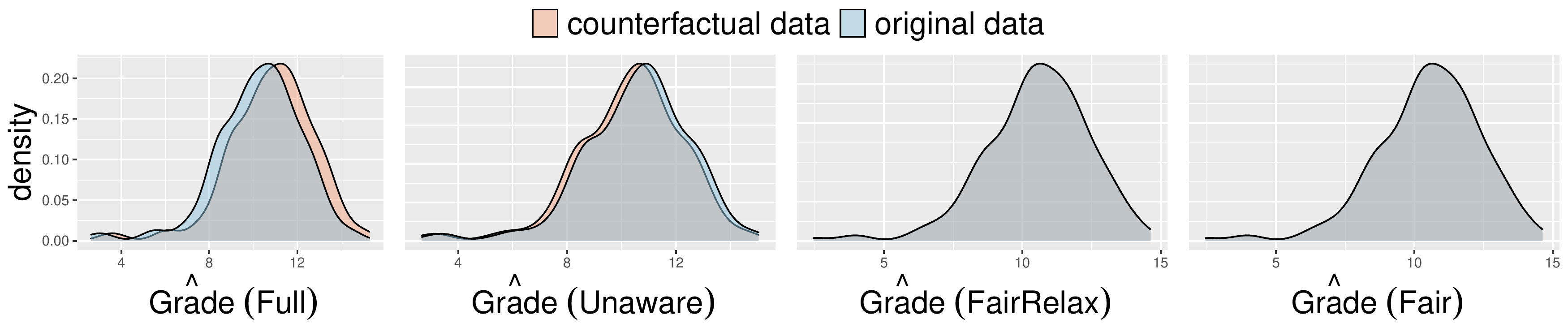}}
\caption{Density plot of the predicted $Grade_{sex \leftarrow a}(u)$ and $Grade_{sex \leftarrow a'}(u)$ for Student dataset.}
\label{fig: Densityplot_unfairness_real_data}
\end{center}
\vskip -0.2in
\end{figure}

%% file: Pages/Conclusion.tex
\section{Conclusion and discussion}\label{sec: Conclusion}

In this paper, we have developed a general approach to achieve counterfactual fairness when the true causal graph is unknown. 
In order to select features that lead to counterfactual fairness, we propose a sufficient and necessary condition and an efficient algorithm to identify the ancestral relations between any distinct vertices on an MPDAG, 
which may be applied to more applications. Furthermore, an intriguing finding is that, under the assumption that the sensitive attribute can only be a root node in the graph, there is no possible descendant of the sensitive attribute, so that the fair features can be selected correctly whatever the true DAG is. Experiments on synthetic and real-world dataset show the effectiveness of our method.

One may claim that the fair prediction could be a function of the descendants of the sensitive attributes by balancing the observables, thus making the effect of the sensitive attribute canceled out. We agree with this opinion.  However, it concerns the structural equations, which are in general unfalsifiable even if interventional data for all variables is available.
In this paper, we do not address the assumptions on structural equations for achieving counterfactual fairness on MPDAGs. As a first step to obtaining a counterfactual fair predictor on an MPDAG, we focus on utilising the property of the causal graph (Level 1 in \citep{kusner2017counterfactual}) — making prediction with the definite non-descendants (and possible descendants) of the sensitive attribute. But our work can also be extended to the case where “cancel out“ could happen on an MPDAG. One possible idea is to learn latent variables by specifying the structural equations (or more relaxed, conditional distributions). However, due to the fact that an MPDAG represents a set of DAGs with different conditional distributions components and thus enjoy different latent space, the intuitive way to enumerate all DAGs is unrealistic. Addressing such issue on an MPDAG is an interesting future direction.

Following prior work on establishing counterfactual fairness \citep{zhang2017causal, chiappa2019path, chikahara2021learning, wu2019counterfactual}, we assumed no selection bias and the presence of confounders throughout this paper. 
Yet, another fundamental assumption in these earlier studies in causal modelling is that the causal graph is known, which offers a new chance for the bias induced by misspecifying the causal DAG. Our work gives the first method to achieve counterfactual fairness without requiring the causal DAG to be specified.
In the presence of selection bias and confounders on a causal DAG, a counterfactual fairness measure degenerates to demographic parity, which is discussed extensively by Fawkes et al. \citep{fawkes2021selection}. In this situation, it is considerably more difficult to provide a clear causal interpretation without specifying the causal DAG.
This gives rise to an exciting topic to research in future work. Exploring causal discovery algorithms to find the ground-truth causal graph in the presence of selection bias and confounders and then achieving counterfactual fairness on partially ancestral graphs \citep{richardson2002ancestral,zhang2008causal}, which would be significantly more difficult, is another possible future direction.



%% file: Pages/Acknowledgement.tex
\section{Acknowledgement}

AZ was supported by Melbourne Research Scholarship from the University of Melbourne. 
SW was supported by ARC DE200101253.
TL was partially supported by Australian Research Council Projects DP180103424, DE-190101473, IC-190100031, DP-220102121, and FT-220100318.
BH was supported by NSFC Young Scientists Fund No. 62006202 and Guangdong Basic and Applied Basic Research Foundation No. 2022A1515011652.
KZ was partially supported by the National Institutes of Health (NIH) under Contract R01HL159805, by the NSF-Convergence Accelerator Track-D award \#2134901, by a grant from Apple Inc., and by a grant from KDDI Research Inc.. 
MG was supported by ARC DE210101624.

%% file: Pages/Appendices/Preliminaries.tex
\section{Preliminaries} \label{Appendix: Preliminaries}


\textbf{Graph and Path.}
Let $p=\langle S=V_{0},..., V_{k}=T\rangle $ be a path in a graph $\mathcal{G}$, $p$ is a \textit{causal path} from $S$ to $T$ if $V_i \rightarrow V_{i+1}$ for all $0 \leq i \leq k-1 $. $p$ is a \textit{possibly causal path} from $S$ to $T$ if no edge $V_i \leftarrow V_{i+1}$ is in $\mathcal{G}$. Otherwise, $p$ is a \textit{non-causal path} in $\mathcal{G}$. A (causal, possibly causal, non-causal) cycle is a (causal, possibly causal, non-causal) path from a vertex to itself.

\textbf{Ancestral Relations.}
If there is $S \rightarrow T$ in $\mathcal{G}$, we say $S$ is a parent of $T$ and $T$ is a child of $S$, denoted by $pa(T,\mathcal{G})$ and $ch(S,\mathcal{G})$, respectively. If there is a causal path from $S$ to $T$, then we say $S$ is an ancestor of $T$ and $T$ is a descendant of $S$, denoted by $an(T,\mathcal{G})$ and $de(S,\mathcal{G})$. If there is a possibly causal path from $S$ to $T$, then we say $S$ is a possible ancestor of $T$ and $T$ is a possible descendant of $S$, denoted by $possAn(T,\mathcal{G})$ and $possDe(S,\mathcal{G})$. As a convention, we regard every node as an ancestor and a descendant of itself.

\textbf{MPDAGs Construction.} Borrowed from \citep{perkovic2017interpreting}, \cref{alg: ConstructMPDAG} summarizes, the way to construct the maximal PDAG $\mathcal{G'}$ from the maximal PDAG $\mathcal{G}$ and backgroud knowledge $\mathcal{B}$, by leveraging Meek's rule in \cref{fig: orientation_rule}. Specifically, here the background knowledge $\mathcal{B}$ is assumed to be the \textit{direct causal information} in the form $S \rightarrow T$, meaning that $S$ is a direct cause of $T$. If Algorithm \ref{alg: ConstructMPDAG} does not return FAIL, then the background knowledge $\mathcal{B}$ and returned maximal PDAG $\mathcal{G'}$ are consistent with the input maximal PDAG $\mathcal{G}$.

\begin{algorithm}
\caption{Construct MPDAG \citep{perkovic2017interpreting, meek1995causal}}
\label{alg: ConstructMPDAG}
\begin{algorithmic}[1]
\State {\bfseries Inputs:} MPDAG $\mathcal{G}$ and Background knowledge $\mathcal{B}$.
\State {\bfseries Output:} MPDAG $\mathcal{G'}$ or FAIL.
\State Let $\mathcal{G'}=\mathcal{G}$;
\While{$\mathcal{B} \neq \emptyset$}
    \State Select an edge $\{S \rightarrow T\}$ in $\mathcal{B}$;
    \State $\mathcal{B} = \mathcal{B} \backslash \{S \rightarrow T\}$;
    \If{$\{S-T\}$ OR $\{S\rightarrow T\}$ is in $\mathcal{G'}$}
        \State Orient $\{S\rightarrow T\}$ in $\mathcal{G'}$;
        \State Orienting edges in $\mathcal{G'}$ following the rules in Figure \ref{fig: orientation_rule} until no edge can be oriented;
    \Else
        \State FAIL;
    \EndIf
\EndWhile
\end{algorithmic}
\end{algorithm}

\begin{figure}[htp]
\begin{center}
\centerline{\includegraphics[width=0.45\columnwidth]{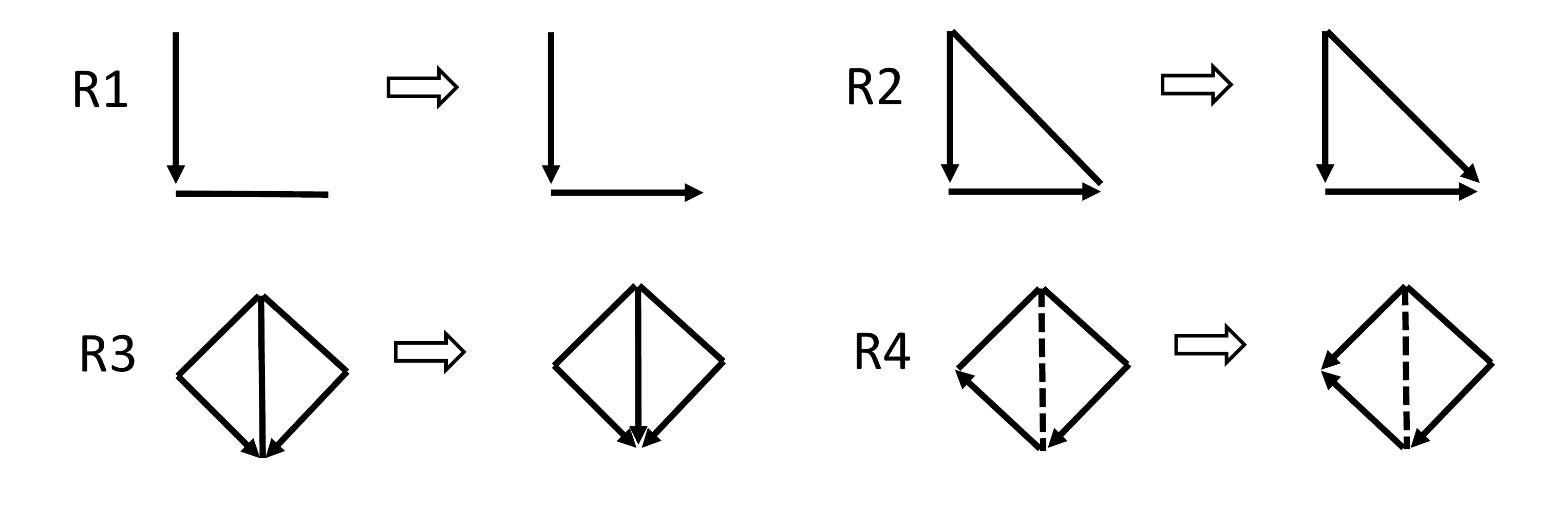}}
\caption{Meek's orientation rules: R1, R2, R3 and R4 \citep{meek1995causal}. For each rule, if the left-hand side graph is an induced subgraph of a PDAG $\mathcal{G}$, orient the undirected edge on it with the direction on the right-hand side.}
\label{fig: orientation_rule}
\end{center}
\end{figure}


\subsection{Existing results}

\begin{lemma}\citep[Lemma B.1]{perkovic2017interpreting} \label{lem: perkovic2017interpreting/Lemma_B.1}
    Let $p=\langle V_{1},..., V_{k}\rangle$ be a b-possibly causal definite status path in an MPDAG $\mathcal{G}$. If there is a node $i \in \{1,...,n-1\}$ such that $V_i \rightarrow V_{i+1}$, then $p(V_i, V_k)$ is a causal path in $\mathcal{G}$.
\end{lemma}

\begin{lemma}\citep[Lemma 3.6]{perkovic2017interpreting}
\label{lem: perkovic2017interpreting/Lemma_3.6}
    Let $S$ and $T$ be distinct nodes in an MPDAG $\mathcal{G}$. If $p$ is a b-possibly causal path from $S$ to $T$ in $\mathcal{G}$, then a subsequence $p^*$ of $p$ forms a b-possibly causal unshielded path from $S$ to $T$ in $\mathcal{G}$.
\end{lemma}

Let $X$ be a variable in an MPDAG $\mathcal{G}$, $\mathbf{R} \subset{sib(X, \mathcal{G})}$, then we use $\mathcal{G}_{\mathbf{R} \rightarrow X}$ to denote the partially directed graph resulted by orienting $\mathbf{R} \rightarrow X$ and $X \rightarrow sib(X,\mathcal{G}) \backslash \mathbf{R}$ in $\mathcal{G}$.  Fang \& He \citep{fang2020ida} propose the following \cref{theo: fang2020ida/parental_set} to check the existence of $\mathcal{G}_{\mathbf{R} \rightarrow X}$.
\begin{theorem}\citep[Theorem 1]{fang2020ida}
\label{theo: fang2020ida/Theorem_1}
\label{theo: fang2020ida/parental_set}
    Let $\mathcal{G}$ be an MPDAG consistent with a CPDAG $\mathcal{G^*}$. For any vertex $X$ and $\mathbf{R}\subset{sib(X, \mathcal{G})}$, the following three statements are equivalent.
    \begin{itemize}[leftmargin=*]
    \itemsep0em
        \item There is a DAG $\mathcal{D} \in [\mathcal{G}]$ such that $pa(X,\mathcal{D})=\mathbf{R} \cup pa(X, \mathcal{G})$ and $ch(X,\mathcal{D}) = sib(X,\mathcal{G}) \cup ch(X,\mathcal{G}) \backslash \mathbf{R}$.
        \item Compared with $\mathcal{G}$, $\mathcal{G}_{\mathbf{R} \rightarrow X}$ does not introduce any new V-structure collided on X or any directed triangle containing $X$.
        \item The induced subgraph of $\mathcal{G}$ over $\mathbf{R}$ is complete, and there does not exist an $R \in \mathbf{R}$ and a $W \in adj(X, \mathcal{G}) \backslash (\mathbf{R} \cup pa(X, \mathcal{G}))$ such that $W \rightarrow R$.
    \end{itemize}
\end{theorem}

\begin{definition} [Critical Set] \citep[Definition 2]{fang2020ida} 
\label{def: fang2020ida/Critical Set}
    Let $\mathcal G^*$ be a CPDAG. $S$ and $T$ are two distinct vertices in $\mathcal G^*$. The critical set of $S$ with respect to $T$ in $\mathcal G^*$ consists of all adjacent vertices of $S$ lying on at least one chordless possibly causal path from $S$ to $T$.
\end{definition}

\begin{theorem} \citep[Theorem 1]{fang2022local}
\label{theo: fang2021local/CPDAG}
    Suppose that $\mathcal{G^*}$ is a CPDAG, $S$ and $T$ are two distinct vertices in $\mathcal{G^*}$, and $\mathbf{C}$ is the critical set of $S$ with respect to $T$ in $\mathcal{G^*}$. Then, $T$ is a definite descendant of $S$ if and only if $\mathbf{C} \cap{ch(S, \mathcal G^*)} \ne \emptyset$, or $\mathbf{C}$ is non-empty and induces an incomplete subgraph of $\mathcal{G^*}$.
\end{theorem}

\begin{lemma} \citep[Lemma 3.1]{maathuis2009estimating}
\label{lem: maathuis2009estimating}
    Given a CPDAG $\mathcal{G^*}$, a variable $X$, and $R \subset{sib(X, \mathcal{G^*})}$, orienting $R \rightarrow X$ for each $R \in \mathbf{R}$ and $X \rightarrow W$ for each $W \in sib(S,\mathcal{G^*}) \backslash \mathbf{R}$ is consistent with $\mathcal{G^*}$ if and only if new orientations do not introduce v-structures collided on $X$.
\end{lemma}

%% file: Pages/Appendices/Proofs.tex
\section{Detailed proofs}
\subsection{Proof of \cref{lem: definite ancestral relations in Markov equivalent DAGs}}
\label{proof_lem: definite ancestral relations in Markov equivalent DAGs}
\begin{proof}
    First, we prove the sufficiency. Let $\mathcal{D}$ be any underlying DAG $\mathcal{D} \in [\mathcal{G}]$, and $\mathbf{C}$ be the critical set of $S$ with respect to $T$ in $\mathcal{G}$. Suppose $C \in \mathbf{C}$ is a child of $S$ in $\mathcal{D}$, that is $S \rightarrow C$ in $\mathcal{D}$. By the definition of critical set, $C$ lies on a chordless b-possibly causal path $\pi$ from $S$ to $T$ in $\mathcal{G}$. Since $S \rightarrow C$ in $\mathcal{D}$, by \cref{lem: perkovic2017interpreting/Lemma_B.1}, the corresponding path $\pi$ in $\mathcal{D}$ is directed. Therefore, $S$ is an ancestor of $T$ in the underlying DAG. 
    
    Next, we prove the necessity: For another direction, suppose that $S$ is a definite ancestor of $T$ in any underlying DAG $\mathcal{D}$. Let $\pi$ be the shortest causal path from $S$ to $T$ in $\mathcal{D}$, then the corresponding path of $\pi$ in $\mathcal{G}$ is a chordless b-possibly causal path, since if $\pi$ has any chord in $\mathcal{G}$, $\pi$ in $\mathcal{D}$ cannot be the shortest path. Denote the vertex adjacent to $S$ on $\pi$ be $C$, then $C \in \mathbf{C}$ and $C$ is a child of $S$ in the DAG $\mathcal{D}$. Therefore, if $T$ is definite descendant of $S$ in $\mathcal{G}$, then $\mathbf{C}$ always contains a child of $S$ in every DAG $\mathcal{D} \in [\mathcal{G}]$.
\end{proof}

\subsection{Proof of \cref{lem: graphical characteristic of critical set}}
\label{proof_lem: graphical characteristic of critical set}
The proof idea of \cref{lem: graphical characteristic of critical set} is to find the graphical condition in an MPDAG that when $\mathbf{C} \neq \emptyset$, all vertices in $\mathbf{C}$ or a superset of $\mathbf{C}$ can be oriented to $X$ in some Markov equivalent DAG, by utilizing locally valid orientation rules for MPDAGs (\cref{theo: fang2020ida/Theorem_1}). The rules are to check whether a set of variables in an MPDAG can be the parents of a given target. 

In this section, we will first introduce some technical lemmas, and then prove \cref{lem: graphical characteristic of critical set} in \cref{sec: Graphical criterion}.

\subsubsection{Technical lemmas}
In this section, we introduce some technical lemmas that are useful in the proof of \cref{lem: graphical characteristic of critical set}.

\begin{lemma}\label{lem: complete2complete}
    Let $\mathcal{G}$ be an MPDAG. For any vertex $X$ in $\mathcal{G}$ and $\mathbf{R} \subset sib(X,\mathcal{G})$, if $\mathbf{R}$ induces a complete subgraph of $\mathcal{G}$ and there exists a $R \in \mathbf{R}$ and a $W \in {adj(X,\mathcal{G}) \backslash{(\mathbf{R} \cup pa(X,\mathcal{G}))}}$ such that $W \rightarrow R$, then $\mathbf{R} \cup W$ induces a complete subgraph of $\mathcal{G}$ and $W \in {sib(X,\mathcal{G}) \backslash{\mathbf{R}}}$.
\end{lemma}
\begin{proof}
     Suppose for a contradiction that $\mathbf{R} \cup W$ induces an incomplete subgraph of $\mathcal{G}$. Since $\mathbf{R}$ induces a complete subgraph of $\mathcal{G}$, that means some vertex $R' \in \mathbf{R}$ is not adjacent with $W$. As $W \in {adj(X,\mathcal{G}) \backslash{(\mathbf{R} \cup pa(X,\mathcal{G}))}}$, the node $W$ can be a child or a sibling of $X$.
     \begin{enumerate}[leftmargin=6.5mm]
        \itemsep0em
         \item [(1)] For the first case, as in \cref{fig: forproof_case1}, if $W$ is a child of $X$ in $\mathcal{G}$, since $W \rightarrow R$, then $X - R$ can be oriented by Rule 2 in Meek's criteria as $X \rightarrow R$, which contradicts that $R$ is a sibling of $X$;
         \item [(2)] For the second case, if $W$ is a sibling of $X$ in $\mathcal{G}$ and $W \rightarrow R$, then the edge between $R$ and $R'$ can not be an undirected edge in $\mathcal{G}$, since $R'$ and $W$ are not adjacent. $R \rightarrow R'$ can be oriented by Meek's Rule 2, or $R \leftarrow R'$ and $\langle R', R, W \rangle$ is a v-structure collided on $R$. For the former case, if $R \rightarrow R'$ is in $\mathcal{G}$ as in \cref{fig: forproof_case2}, then $X \rightarrow R'$ can be oriented by Rule 4, which contradicts that $R'$ is a sibling of $X$ in $\mathcal{G}$. For the latter case in \cref{fig: forproof_case3}, if $W$ is a sibling of $X$ in $\mathcal{G}$ and $W \rightarrow R \leftarrow R'$, since $R'$ and $W$ are not adjacent, then $X - R$ can be oriented as $X \rightarrow R$ in $\mathcal{G}$ by Meek's Rule 3, which contradicts that $R$ is a sibling of $X$. Therefore, there does not exist any vertex $R' \in \mathbf{R}$ not adjacent with $W$, so $\mathbf{R} \cup W$ induces a complete subgraph of $\mathcal{G}$.
     \end{enumerate}
\end{proof} 

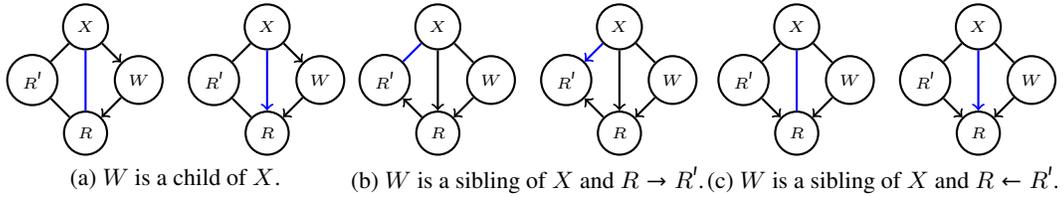
\begin{figure}[ht]
\centering
\subfloat[$W$ is a child of $X$.]{%
    \label{fig: forproof_case1}
    \begin{tikzpicture}[node distance={10mm}, thick, main/.style = {draw, circle}] 
    \tikzstyle{every node}=[font=\tiny]
    \node[main] (1) {$X$}; 
    \node[main] (2) [below left of=1] {$R'$};
    \node[main] (3) [below right of=1] {$W$}; 
    \node[main] (4) [below right of=2] {$R$};
    \node[main] (6) [right of=3] {$R'$};
    \node[main] (5) [above right of=6]{$X$}; 
    \node[main] (7) [below right of=5] {$W$}; 
    \node[main] (8) [below right of=6] {$R$};
    \draw(1) -- (2); 
    \draw[->] (1) -- (3);
    \draw[blue] (1) -- (4);
    \draw[->] (3) -- (4);
    \draw(4) -- (2); 
    \draw(5) -- (6);
    \draw[->] (5) -- (7);
    \draw[blue, ->] (5) -- (8);
    \draw[->] (7) -- (8);
    \draw(8) -- (6);
    \end{tikzpicture} }
\subfloat[$W$ is a sibling of $X$ and $R \rightarrow R'$.]{%
    \label{fig: forproof_case2}
    \begin{tikzpicture}[node distance={10mm}, thick, main/.style = {draw, circle}] 
    \tikzstyle{every node}=[font=\tiny]
    \node[main] (1) {$X$}; 
    \node[main] (2) [below left of=1] {$R'$};
    \node[main] (3) [below right of=1] {$W$}; 
    \node[main] (4) [below right of=2] {$R$};
    \node[main] (6) [right of=3] {$R'$};
    \node[main] (5) [above right of=6]{$X$}; 
    \node[main] (7) [below right of=5] {$W$}; 
    \node[main] (8) [below right of=6] {$R$};
    \draw[blue] (1) -- (2);
    \draw (1) -- (3);
    \draw[->] (1) -- (4);
    \draw[->] (3) -- (4);
    \draw[->] (4) -- (2);
    \draw[blue, ->] (5) -- (6);
    \draw (5) -- (7);
    \draw[->] (5) -- (8);
    \draw[->] (7) -- (8);
    \draw[->] (8) -- (6);
    \end{tikzpicture} }
\subfloat[$W$ is a sibling of $X$ and $R \leftarrow R'$.]{
    \label{fig: forproof_case3}
    \begin{tikzpicture}[node distance={10mm}, thick, main/.style = {draw, circle}] 
    \tikzstyle{every node}=[font=\tiny]
    \node[main] (1) {$X$}; 
    \node[main] (2) [below left of=1] {$R'$};
    \node[main] (3) [below right of=1] {$W$}; 
    \node[main] (4) [below right of=2] {$R$};
    \node[main] (6) [right of=3] {$R'$};
    \node[main] (5) [above right of=6]{$X$}; 
    \node[main] (7) [below right of=5] {$W$}; 
    \node[main] (8) [below right of=6] {$R$};
    \draw (1) -- (2);
    \draw (1) -- (3);
    \draw[blue] (1) -- (4);
    \draw[->] (2) -- (4);
    \draw[->] (3) -- (4);
    \draw (5) -- (6);
    \draw (5) -- (7);
    \draw[blue, ->] (5) -- (8);
    \draw[->] (6) -- (8);
    \draw[->] (7) -- (8);
    \end{tikzpicture} }
\caption{The three cases discussed in the proof of \cref{lem: complete2complete}. For each case, the the blue undirected edge on the left-hand side subgraph will be oriented as the blue edge on the right-hand side subgraph.}
\label{fig: forproof}
\end{figure}

\begin{lemma}\label{lem: complex2simple}
In an MPDAG $\mathcal{G}$, for any vertex $X$, there exists $\mathbf{H} \subseteq{sib(X,\mathcal{G})}$ that induces a complete subgraph of $\mathcal{G}$ if and only if there exists some $\mathbf{R}$ that $\mathbf{H} \subseteq \mathbf{R} \subseteq{sib(X,\mathcal{G})}$, such that there is a DAG $\mathcal{D} \in [\mathcal{G}]$ that $pa(X,\mathcal{D})=\mathbf{R} \cup pa(X,\mathcal{G})$ and $ch(X,\mathcal{D})=sib(X,\mathcal{G}) \cup ch(X, \mathcal{G}) \backslash \mathbf{R}$.
\end{lemma}
\begin{proof}
    According to \cref{theo: fang2020ida/Theorem_1}, in an MPDAG $\mathcal{G}$, for any vertex $X$ and $\mathbf{R} \subset sib(X,\mathcal{G})$, the following two statements are equivalent:
    \begin{enumerate}[leftmargin=6.5mm]
    \itemsep0em
        \item[(1)] There is a DAG $\mathcal{D} \in [\mathcal{G}]$ such that $pa(X,\mathcal{D})=\mathbf{R} \cup pa(X,\mathcal{G})$ and $ch(X,\mathcal{D})=sib(X,\mathcal{G}) \cup ch(X, \mathcal{G}) \backslash \mathbf{R}$.
        \item[(2)] The induced subgraph of $\mathcal{G}$ over $\mathbf{R}$ is complete, and there does not exist an $R \in \mathbf{R}$ and a $W \in {adj(X,\mathcal{G}) \backslash{(\mathbf{R} \cup pa(X,\mathcal{G}))}}$ such that $W \rightarrow R$.
    \end{enumerate}

    Therefore, \cref{lem: complex2simple} can also be stated as: In an MPDAG $\mathcal{G}$, for any vertex $X$, there exists $\mathbf{H} \subseteq{sib(X,\mathcal{G})}$ that induces a complete subgraph of $\mathcal{G}$ if and only if there exists some $\mathbf{R}$ that $\mathbf{H} \subseteq \mathbf{R} \subseteq{sib(X,\mathcal{G})}$, such that the induced subgraph of $\mathcal{G}$ over $\mathbf{R}$ is complete, and there does not exist an $R \in \mathbf{R}$ and a $W \in {adj(X,\mathcal{G}) \backslash{(\mathbf{R} \cup pa(X,\mathcal{G}))}}$ such that $W \rightarrow R$.
    
    The proof of sufficiency is straightforward. That set $\mathbf{R}$ induces a complete subgraph can ensure that any subset of $\mathbf{R}$ induces a complete subgraph of $\mathcal{G}$.
    
    Next, we prove the necessity. If $\mathbf{H} \subseteq \mathbf{R}$ is complete and there does not exist an $H \in \mathbf{H}$ and a $W \in {adj(X,\mathcal{G}) \backslash{(\mathbf{H} \cup pa(X,\mathcal{G}))}}$ such that $W \rightarrow H$, then $\mathbf{H}$ satisfies the condition and we are done. Otherwise, if $\mathbf{H}$ is complete and there exists an $H \in \mathbf{H}$ and a $W \in {adj(X,\mathcal{G}) \backslash{(\mathbf{H} \cup pa(X,\mathcal{G}))}}$ such that $W \rightarrow H$, by \cref{lem: complete2complete}, $(\mathbf{H} \cup W) \subseteq \mathbf{R}$ induces a complete subgraph of $\mathcal{G}$. Similarly, if $\mathbf{H} \cup W$ is complete and there does not exist an $H' \in \mathbf{H} \cup W$ and a $W' \in {adj(X,\mathcal{G}) \backslash{(\mathbf{H} \cup W \cup pa(X,\mathcal{G}))}}$ such that $W' \rightarrow H'$, then $\mathbf{H} \cup W$ satisfies the condition and we are done. Otherwise, $(\mathbf{H} \cup W \cup W') \subseteq \mathbf{R}$ induces a complete subgraph of $\mathcal{G}$. Following this derivation, either we are done or we will end with the result that $sib(X,\mathcal{G})$ is complete. For the latter situation, by \cref{theo: fang2020ida/Theorem_1}, since orienting every vertex in $sib(X,\mathcal{G})$ towards $X$ does not introduce any new V-structure collided on $X$ or any directed triangle containing $X$. In this case, $\mathbf{R}=sib(X,\mathcal{G})$ meets the left hand side and we are done.
\end{proof}

\subsubsection{Proof of \cref{lem: graphical characteristic of critical set}}
\begin{proof}
    We first show the necessity. By \cref{def: Critical Set}, $\mathbf{C} \subseteq sib(S,\mathcal{G}) \cup ch(S,\mathcal{G})$. Let $\mathcal{D} \in [\mathcal{G}]$ be an arbitrary DAG. If $\mathbf{C} \cap{ch(S, \mathcal{D})} = \emptyset$ and $\mathbf{C} \ne \emptyset$, then $\mathbf{C} \subseteq pa(S, \mathcal{D})$, and thus $\mathbf{C} \subseteq sib(S,\mathcal{G})$. Denote by $\mathbf{R} = sib(S, \mathcal{G}) \cap pa(S,\mathcal{D})$, we have $\mathbf{C} \subseteq{\mathbf{R}} \subseteq{sib(S,\mathcal{G})}$ and $\mathbf{C} \subseteq{\mathbf{R}} \subseteq{pa(S,\mathcal{D})}$. \cref{theo: fang2020ida/Theorem_1} proved that a non-empty subset $\mathbf{R}$ of $sib(S,G)$ can be a part of $S$'s parent set in some equivalent DAG if and only if $\mathbf{R}$ induces a complete subgraph of $\mathcal{G}$, and there does not exist a set $R \in \mathbf{R}$ and a $W \in adj(S,\mathcal{G})\backslash (\mathbf{R} \cup pa(S,\mathcal{G}))$ such that $W \rightarrow R$. Therefore, as a subset of $\mathbf{R}$, $\mathbf{C}$ induces a complete subgraph of $\mathcal{G}$. This completes the proof of necessity.
    
    We next prove the sufficiency. If $\mathbf{C}=\emptyset$, then it is straightforward that $\mathbf{C} \cap{ch(S, \mathcal D)} = \emptyset$ for some $\mathcal{D} \in [\mathcal{G}]$. Now assume $\mathbf{C} \ne \emptyset$ and $\mathbf{C} \cap{ch(S, \mathcal G)} = \emptyset$. As $\mathbf{C} \subseteq sib(S,\mathcal{G}) \cup ch(S,\mathcal{G})$, we have $\mathbf{C} \subseteq sib(S,\mathcal{G})$. Since $\mathbf{C}$ induces a complete subgraph of $\mathcal{G}$, by \cref{lem: complex2simple}, there exists $\mathbf{R}$, $\mathbf{C} \subseteq{\mathbf{R}} \subseteq{sib(S,\mathcal{G})}$, that there is a DAG $\mathcal{D} \in [\mathcal{G}]$ such that $pa(S,\mathcal{D})=\mathbf{R} \cup pa(S,\mathcal{G})$ and $ch(S,\mathcal{D})=sib(S,\mathcal{G}) \cup ch(S, \mathcal{G}) \backslash \mathbf{R}$. As $\mathbf{R} \subseteq pa(S,\mathcal{D})$ and $\mathbf{C} \subseteq{\mathbf{R}}$, $\mathbf{C} \subseteq pa(S,\mathcal{D})$ and thus $\mathbf{C} \cap{ch(S, \mathcal{D})} = \emptyset$.
\end{proof}

\subsection{Proof of \cref{theo}}
   \cref{theo} is closely related to \cref{theo: fang2021local/CPDAG} for CPDAGs from \citep{fang2022local}. Since CPDAG is a special case of MPDAG, all results for MPDAGs works for CPDAGs. Although the condition provided by these two theorems are the same, they based on different theoretical results on locally valid orientation rules for CPDAGs and MPDAGs. The one for CPDAGs is mainly based on \cref{lem: maathuis2009estimating} and for MPDAGs, it is mainly based on \cref{theo: fang2020ida/Theorem_1}.
\begin{proof}
    \cref{fig: Proof structure of theo} shows how all lemmas fit together to prove the \cref{theo}. To decide whether $T$ is a definite descendant of $S$ in an MPDAG $\mathcal{G}$, \cref{lem: definite ancestral relations in Markov equivalent DAGs} provides a sufficient and necessary condition on a graphical characteristic of $\mathbf{C}$ on every DAG $\mathcal{D} \in [\mathcal{G}]$, which is then further explored by \cref{lem: graphical characteristic of critical set} to a graphical characteristic of $\mathbf{C}$ on the corresponding MPDAG $\mathcal{G}$.
    Following from \cref{lem: definite ancestral relations in Markov equivalent DAGs} and \cref{lem: graphical characteristic of critical set}, we have the desired sufficient and necessary condition to check whether $T$ is a definite descendant of $S$ in an MPDAG $\mathcal{G}$ on the graphical characteristic of $\mathbf{C}$.
\end{proof}

\begin{figure}[htp]
\vskip -0.2in
\begin{center}
    \begin{tikzpicture}[node distance={25mm}, thick, main/.style = {draw}] 
    \node[] (1) {\textbf{\cref{theo}}}; 
    \node[] (2) [right of=1] {\cref{lem: definite ancestral relations in Markov equivalent DAGs}};
    \node[] (3) [left of=1] {\cref{lem: graphical characteristic of critical set}}; 
    \node[] (4) [left of=3] {\cref{lem: complex2simple}};
    \node[] (5) [left of=4] {\cref{lem: complete2complete}};
    \draw[->] (2) -- (1); 
    \draw[->] (3) -- (1);
    \draw[->] (4) -- (3);
    \draw[->] (5) -- (4);
    \end{tikzpicture}
\caption{Proof structure of \cref{theo}}
\label{fig: Proof structure of theo}
\end{center}
\end{figure}


\subsection{Proof of \cref{lem: ds and chordless}}
\label{proof_lem: ds and chordless}
\begin{proof}
    Suppose $p = \left \langle S=V_0,...,V_k = T \right \rangle$ is a chordless path from $S$ to $T$. We will show that $p$ is of definite status by showing that every vertex on $p$ is of definite status. Any triple $\langle V_{i-1}, V_{i}, V_{i+1} \rangle $ on $p$ with $i \in \{1,...,k-1\}$ can be in the form: (1) $V_{i-1} \rightarrow V_{i} \leftarrow V_{i+1}$; or (2) $V_{i}\rightarrow V_{i+1}$ or $V_{i-1}\rightarrow V_{i}$ on $p$ or $V_{i-1} - V_{i} - V_{i+1}$ is a subpath of $p$ and $V_{i-1}$ is not adjacent with $V_{i+1}$. In the former case, $V_{i}$ is a collider; in the latter case, $V_{i}$ is a definite non-collider. The triple cannot be in the form $V_{i-1} \rightarrow V_{i} - V_{i+1}$ or $V_{i-1} - V_{i} \leftarrow V_{i+1}$, since the undirected edge can be oriented departs $V_{i}$ by Rule 1 in Meek's criterion. Therefore, every vertex on $p$ is of definite status. Thus, we completes the proof that $p$ is of definite status.
\end{proof}

\subsection{Proof of \cref{prop: obtain critical set by ds path}}
\label{proof_prop: obtain critical set by ds path}
\begin{proof}
    If all definite status b-possibly causal path from $S$ to $T$ are chordless, then by \cref{lem: ds and chordless}, $\mathbf{F_{ST}}$ is exactly $\mathbf{C_{ST}}$. Suppose that there are definite status b-possibly causal path from $S$ to $T$ with chords. We first prove that $\mathbf{C_{ST}} \subseteq \mathbf{F_{ST}}$. By the definition of critical set, for any $C \in \mathbf{C_{ST}}$, there is a chordless possibly causal path $p$ from $X$ to $Y$ on which $C$ is adjacent to $S$. By \cref{lem: ds and chordless}, $p$ is also a definite status b-possibly causal path from $S$ to $T$ without any chord. Therefore, as an adjacent vertex of $S$ on $p$, $C \in \mathbf{F_{ST}}$ as well. Since $C$ is an arbitrary vertex in $\mathbf{C_{ST}}$, $\mathbf{C_{ST}} \subseteq \mathbf{F_{ST}}$. Then we prove $\mathbf{F_{ST}} \subseteq \mathbf{C_{ST}}$. For any $F \in \mathbf{F_{ST}}$, there is a definite status b-possibly causal path $p^{*}$ from $S$ to $T$ in $\mathcal{G}$ that there is no chord with $S$ as an endpoint, on which $F$ is adjacent to $X$. By \cref{lem: perkovic2017interpreting/Lemma_3.6}, some subsequence of $p^{*}$ forms a chordless b-possibly causal path $p^{\ast \ast}$ from $S$ to $T$. As on $p^{*}$, there is no chord with $S$ as an endpoint, the chordless b-possibly causal path $p^{\ast \ast}$ must start with the edge $S - F ...$ or $S \rightarrow F ...$. Therefore, $F \in \mathbf{C_{ST}}$. Since $F$ is an arbitrary vertex in $\mathbf{F_{ST}}$, $\mathbf{F_{ST}} \subseteq \mathbf{C_{ST}}$. This completes the proof of \cref{prop: obtain critical set by ds path}.
\end{proof}

\subsection{Proof of \cref{prop: no possible ancestral relations in CF in MPDAGs}}
\label{proof_prop: no possible ancestral relations in CF in MPDAGs}
\begin{proof}
    Suppose there is a possible descendant of $A$ in $\mathcal{G}$, which is denoted by $B$. Then there is a b-possibly causal path $p$ from $A$ to $B$. By \cref{lem: perkovic2017interpreting/Lemma_3.6}, a subsequence $p^{*}$ of $p$ forms a b-possibly causal unshielded path from $A$ to $B$. Suppose $p^{*} = \left \langle A=V_0,...,V_k = B \right \rangle$, \cref{assum: A is a root node} implies $A \rightarrow V_1$. By \cref{lem: perkovic2017interpreting/Lemma_B.1}, $p^{*}$ is a causal path from $A$ to $B$ in $\mathcal{G}$. Therefore, $B$ is a definite descendant of $A$.
\end{proof}

%% file: Pages/Appendices/SuppMethod.tex
\section{An illustration example for \cref{theo}} \label{Appendix: An illustration Example}

\textbf{Example.} Consider the MPDAG $\mathcal{G}$ in \cref{fig: eg1} and the node $A$. We show the ancestral relations between $A$ and any other nodes. In $\mathcal{G}$, it is obvious that $B,C,D$ and $H$ are possible descendants of $A$. That is because $B$ is adjacent with $A$, and the critical set of $B$ with respect to $A$ is itself, which induces a complete subgraph of $\mathcal{G}$. The same conclusion can be drawn for $C,D$, and $H$.  Node $E$ is also a possible descendant of $A$, since the chordless possibly causal path from $A$ to $E$ are $A-B-E$ and $A-C-E$, the critical set of $A$ with respect to $E$ is \{$B,C$\}. As the induced subgraph of $\mathcal{G}$ over \{$B,C$\} is complete, by \cref{theo}, $E$ is not a definite descendant of $A$, so it is a possible descendant of $A$. For $F$, the chordless possibly causal path from $A$ to $F$ are $A-B-F$, $A-C-F$ and $A-D-F$, thus the critical set of $A$ with respect to $F$ is \{$B,C,D$\}. Since the corresponding induced subgraph is incomplete, by \cref{theo}, $F$ is a definite descendant of $A$.

\begin{figure}[htp]
\begin{center}
\begin{tikzpicture}[node distance={13mm}, thick, main/.style = {draw, circle}] 
\node[main] (1) {$A$}; 
\node[main] (2) [below left of=1] {$B$};
\node[main] (3) [below right of=2] {$C$}; 
\node[main] (4) [below right of=3] {$D$};
\node[main] (5) [below left of=3] {$E$}; 
\node[main] (6) [below right of=5] {$F$};
\node[main] (7) [above right of=4] {$H$};
\draw (1) -- (2);
\draw (1) -- (3);
\draw (1) -- (4);
\draw (1) -- (7);
\draw (2) -- (3);
\draw (2) -- (5);
\draw[->] (2) to [out=215,in=180,looseness=1.2] (6);
\draw (3) -- (4);
\draw (3) -- (5);
\draw[->] (3) -- (6);
\draw[->] (3) -- (7);
\draw[->] (4) -- (6);
\draw[->] (5) -- (6);
\end{tikzpicture} 
\caption{An MPDAG $\mathcal{G}$ for illustrating ancestral relations of the node $A$ with any other nodes. The node $B,C,D,E$ and $H$ are possible descendants of $A$; node $F$ is a definite descendant of $A$.}
\label{fig: eg1}
\end{center}
\end{figure}
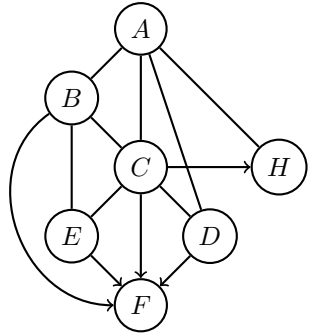

\section{\cref{alg: FindingCriticalSet} and detailed explanation} \label{Appendix: Detailed Explaination of algo}

In \cref{alg: FindingCriticalSet}, every b-possibly causal path of definite status (Line 12-13) on the way starting from $S$ to $\tau$ without any chord ending in $S$ (Line 14) is recorded in a queue $\mathcal{Q}$ as a triple $(\alpha,\phi,\tau)$, where $\alpha$ is the node lying immediately after $S$ and $\phi$ is the node that lie immediately before $\tau$ on the path. If $\tau$ is exactly $Y$, we add $\alpha$ to the critical set $\mathbf{C}$ and remove from $\mathcal{Q}$ all triples where the first element is $\alpha$, that is, we stop enumerating the required paths on which the node adjacent with $S$ is $\alpha$ (Line 8-9). Otherwise, we extend the path to the adjacencies of $\tau$, $\beta$, so that the path from $S$ to $\beta$ is still a b-possibly causal path of definite status without any chord ending in $S$ and then we add the corresponding triples to the queue $\mathcal{Q}$ (Line 11-16). In this algorithm, $\mathcal{H}$ is introduced to store the visited triples, and to avoid visiting the same triple twice.

\section{Algorithm in \cref{sec: General Case}} \label{Appendix: Additional Algorithms}

\begin{algorithm}
\begin{algorithmic}[1]
\caption{Identify the type of ancestral relation of $S$ with respect to all the other vertices in an MPDAG}
\label{alg: ancestral relation between X and all the other nodes}
\State \textbf{Input:} MPDAG $\mathcal{G}$, a variable $S$ in $\mathcal{G}$.
\State \textbf{Output:} The type of ancestral relation between $S$ and all the other vertices in $\mathcal{G}$.
\For{each node $W$ in $\mathcal{G}$}
    \State Identify the type of ancestral relation of $S$ with respect to $W$ in $\mathcal{G}$ by \cref{alg: ancestral relation between X and Y}.
\EndFor
\State \textbf{Return} the set of definite descendants, possible descendants and definite non-descendants of $S$ in $\mathcal{G}$.
\end{algorithmic}
\end{algorithm}

%% file: Pages/Appendices/SuppExp.tex
\newpage
\section{Supplementary experimental results}

\subsection{Causal graphs for one simulation} \label{Appendix: A randomly generated causal graph example}
\begin{figure}[ht]
\centering
\subfloat[DAG $\mathcal{D}$]{%
        \begin{tikzpicture}[node distance={10mm}, thick, main/.style = {draw, circle}]
        \tikzstyle{every node}=[font=\tiny]
        \node[main] (1) {$A$}; 
        \node[main] (10) [below of=1] {$K$};
        \node[main] (5) [below left of=1] {$E$};
        \node[main] (3) [below right of=1] {$C$};
        \node[main] (7) [above left of=5] {$H$};
        \node[main] (9) [below of=10] {$Y$};
        \node[main] (4) [right of=3] {$D$};
        \node[main] (2) [above of=4] {$B$};
        \node[main] (6) [right of=4] {$F$};
        \node[main] (8) [below of=4] {$I$};
        \draw[->] (1) -- (3);
        \draw[->] (1) -- (5);
        \draw[->] (1) -- (7);
        \draw[->] (1) -- (10);
        \draw[->] (2) -- (4);
        \draw[->] (2) -- (6);
        \draw[->] (3) -- (8);
        \draw[->] (3) -- (9);
        \draw[->] (4) -- (8);
        \draw[<-] (5) -- (7);
        \draw[->] (5) -- (9);
        \draw[->] (5) -- (10);
        \draw[->] (6) -- (8);
        \draw[->] (7) to [out=270,in=180,looseness=1.2] (9);
        \draw[->] (8) -- (9);
        \end{tikzpicture} }
\subfloat[CPDAG $\mathcal{G^*}$]{%
        \begin{tikzpicture}[node distance={10mm}, thick, main/.style = {draw, circle}]
        \tikzstyle{every node}=[font=\tiny]
        \node[main] (1) {$A$}; 
        \node[main] (10) [below of=1] {$K$};
        \node[main] (5) [below left of=1] {$E$};
        \node[main] (3) [below right of=1] {$C$};
        \node[main] (7) [above left of=5] {$H$};
        \node[main] (9) [below of=10] {$Y$};
        \node[main] (4) [right of=3] {$D$};
        \node[main] (2) [above of=4] {$B$};
        \node[main] (6) [right of=4] {$F$};
        \node[main] (8) [below of=4] {$I$};
        \draw (1) -- (3);
        \draw (1) -- (5);
        \draw (1) -- (7);
        \draw (1) -- (10);
        \draw (2) -- (4);
        \draw (2) -- (6);
        \draw[->] (3) -- (8);
        \draw[->] (3) -- (9);
        \draw[->] (4) -- (8);
        \draw (5) -- (7);
        \draw[->] (5) -- (9);
        \draw (5) -- (10);
        \draw[->] (6) -- (8);
        \draw[->] (7) to [out=270,in=180,looseness=1.2] (9);
        \draw[->] (8) -- (9);
        \end{tikzpicture} }
\subfloat[MPDAG $\mathcal{G}$]{
        \begin{tikzpicture}[node distance={10mm}, thick, main/.style = {draw, circle}]
        \tikzstyle{every node}=[font=\tiny]
        \node[main] (1) {$A$}; 
        \node[main] (10) [below of=1] {$K$};
        \node[main] (5) [below left of=1] {$E$};
        \node[main] (3) [below right of=1] {$C$};
        \node[main] (7) [above left of=5] {$H$};
        \node[main] (9) [below of=10] {$Y$};
        \node[main] (4) [right of=3] {$D$};
        \node[main] (2) [above of=4] {$B$};
        \node[main] (6) [right of=4] {$F$};
        \node[main] (8) [below of=4] {$I$};
        \draw (1) -- (3);
        \draw (1) -- (5);
        \draw (1) -- (7);
        \draw (1) -- (10);
        \draw (2) -- (4);
        \draw (2) -- (6);
        \draw[->] (3) -- (8);
        \draw[->] (3) -- (9);
        \draw[->] (4) -- (8);
        \draw (5) -- (7);
        \draw[->] (5) -- (9);
        \draw[->] (5) -- (10);
        \draw[->] (6) -- (8);
        \draw[->] (7) to [out=270,in=180,looseness=1.2] (9);
        \draw[->] (8) -- (9);
        \end{tikzpicture} }
\caption{(a) is one of the generated DAG $\mathcal{D}$ with $10$ nodes and $10$ directed edges; (b) is the corresponding CPDAG $\mathcal{G^*}$; (c) is the corresponding MPDAG $\mathcal{G}$ following Meek's rule, with the background knowledge that $E$ is a direct cause of $K$. The randomly selected sensitive attribute is represented by $A$ and the outcome attribute is $Y$. \cref{alg: ancestral relation between X and all the other nodes} detects the ancestral relations in MPDAG $\mathcal{G}$: the definite non-descendants of the sensitive attributes are $\{B,D,F\}$, the possible descendants are $\{C,H,E,K,I\}$, and there is no definite descendants of $A$ in $\mathcal{G}$.}
\label{fig: SimulationGraph}
\end{figure}
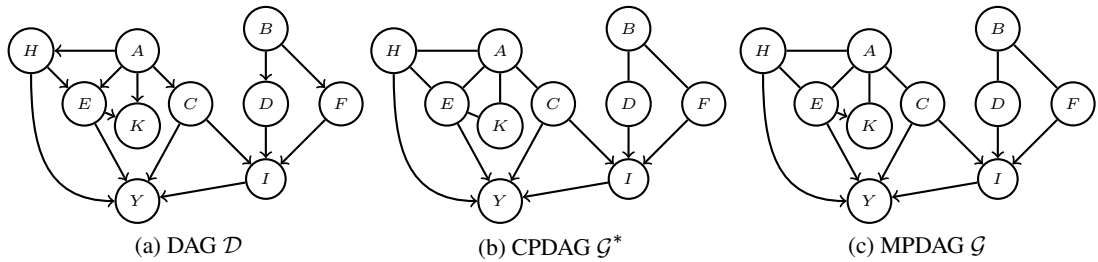



\newpage
\subsection{Causal graphs for the Student Dataset} \label{Appendix: Causal graphs for Student Dataset}

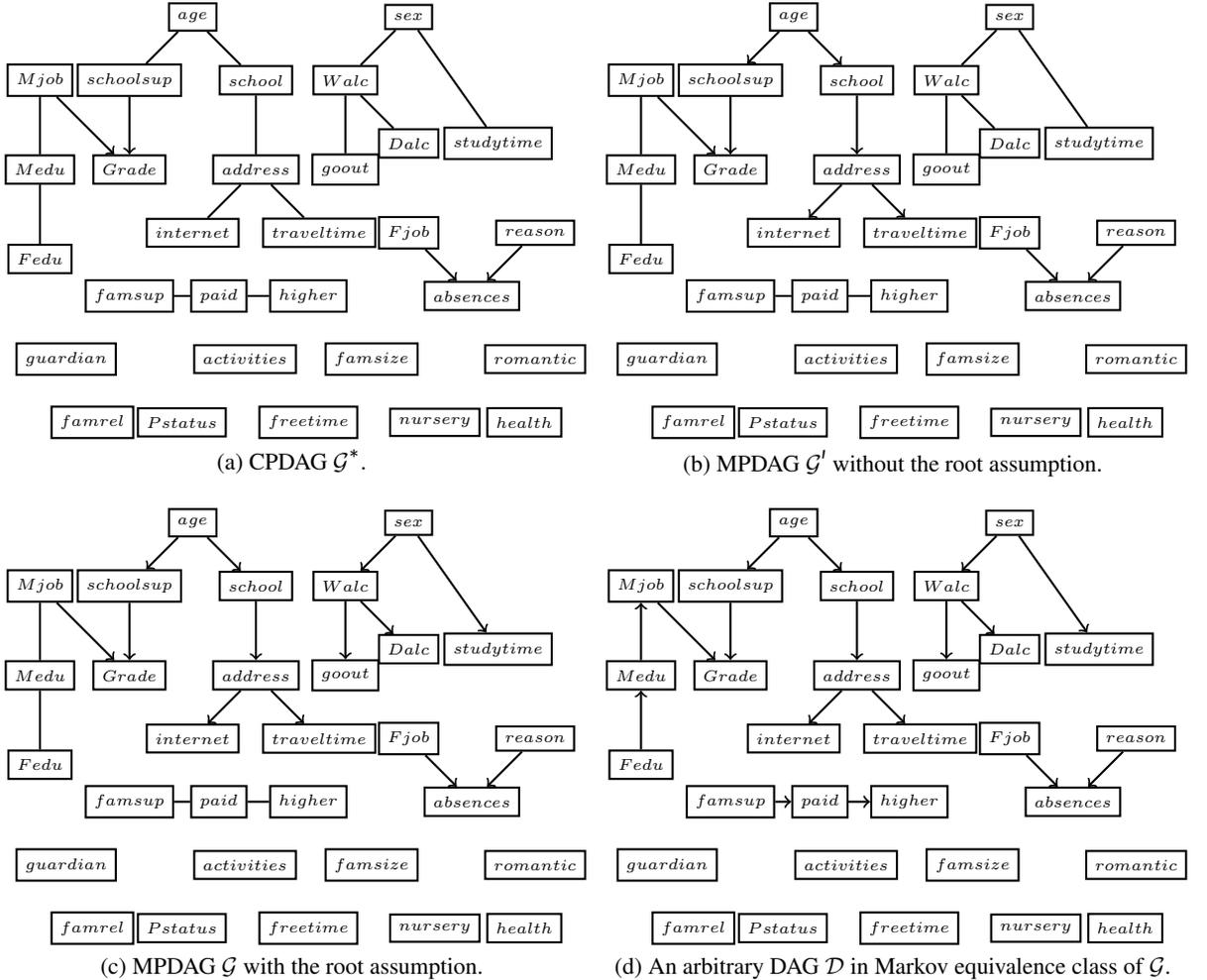
\begin{figure}[ht]
\centering
\subfloat[CPDAG $\mathcal{G^*}$.]{
    \label{fig: real_data_cpdag}
    \begin{tikzpicture}[node distance={12mm}, thick, main/.style = {draw, rectangle}]
    \tikzstyle{every node}=[font=\tiny]
    \node[main] (3) {$Mjob$}; 
    \node[main] (2) [below of=3] {$Medu$};
    \node[main] (1) [below of=2]{$Fedu$}; 
    \node[main] (4) [right of=2] {$Grade$};
    \node[main] (5) [above of=4] {$schoolsup$}; 
    \node[main] (6) [above right of=5] {$age$};
    \node[main] (7) [below right of=6] {$school$};
    \node[main] (8) [below of=7] {$address$};
    \node[main] (9) [below left of=8] {$internet$}; 
    \node[main] (10) [below right of=8] {$traveltime$};
    \node[main] (13) [right of=8] {$goout$};
    \node[main] (12) [above of=13] {$Walc$};
    \node[main] (14) [below right of=12] {$Dalc$};
    \node[main] (11) [above right of=12] {$sex$}; 
    \node[main] (15) [right of=14] {$studytime$};
    \node[main] (16) [below of=14] {$Fjob$};
    \node[main] (18) [below right of=16] {$absences$};
    \node[main] (17) [above right of=18] {$reason$};
    \node[main] (19) [below left of=9] {$famsup$};
    \node[main] (20) [right of=19] {$paid$};
    \node[main] (21) [right of=20] {$higher$};
    \node[main] (24) [below left of=21] {$activities$};
    \node[main] (25) [below right of=21] {$famsize$};
    \node[main] (27) [below left of=19] {$guardian$};
    \node[main] (23) [below left of=24] {$Pstatus$};
    \node[main] (22) [left of=23] {$famrel$};
    \node[main] (29) [below right of=25] {$nursery$};
    \node[main] (28) [right of=29] {$health$};
    \node[main] (26) [below right of=24] {$freetime$};
    \node[main] (30) [below right of=18] {$romantic$};
    \draw (1) -- (2);
    \draw (2) -- (3);
    \draw[->] (3) -- (4);
    \draw[->] (5) -- (4);
    \draw (6) -- (5);
    \draw (6) -- (7);
    \draw (7) -- (8);
    \draw (8) -- (9);
    \draw (8) -- (10);
    \draw (11) -- (12);
    \draw (12) -- (13);
    \draw (12) -- (14);
    \draw (11) -- (15);
    \draw[->] (16) -- (18);
    \draw[->] (17) -- (18);
    \draw (19) -- (20);
    \draw (20) -- (21);
    \end{tikzpicture} }
\subfloat[MPDAG $\mathcal{G'}$ without the root assumption.]{
    \label{fig: real_data_without_assumption}
    \begin{tikzpicture}[node distance={12mm}, thick, main/.style = {draw, rectangle}]
    \tikzstyle{every node}=[font=\tiny]
    \node[main] (3) {$Mjob$}; 
    \node[main] (2) [below of=3] {$Medu$};
    \node[main] (1) [below of=2]{$Fedu$}; 
    \node[main] (4) [right of=2] {$Grade$};
    \node[main] (5) [above of=4] {$schoolsup$}; 
    \node[main] (6) [above right of=5] {$age$};
    \node[main] (7) [below right of=6] {$school$};
    \node[main] (8) [below of=7] {$address$};
    \node[main] (9) [below left of=8] {$internet$}; 
    \node[main] (10) [below right of=8] {$traveltime$};
    \node[main] (13) [right of=8] {$goout$};
    \node[main] (12) [above of=13] {$Walc$};
    \node[main] (14) [below right of=12] {$Dalc$};
    \node[main] (11) [above right of=12] {$sex$}; 
    \node[main] (15) [right of=14] {$studytime$};
    \node[main] (16) [below of=14] {$Fjob$};
    \node[main] (18) [below right of=16] {$absences$};
    \node[main] (17) [above right of=18] {$reason$};
    \node[main] (19) [below left of=9] {$famsup$};
    \node[main] (20) [right of=19] {$paid$};
    \node[main] (21) [right of=20] {$higher$};
    \node[main] (24) [below left of=21] {$activities$};
    \node[main] (25) [below right of=21] {$famsize$};
    \node[main] (27) [below left of=19] {$guardian$};
    \node[main] (23) [below left of=24] {$Pstatus$};
    \node[main] (22) [left of=23] {$famrel$};
    \node[main] (29) [below right of=25] {$nursery$};
    \node[main] (28) [right of=29] {$health$};
    \node[main] (26) [below right of=24] {$freetime$};
    \node[main] (30) [below right of=18] {$romantic$};
    \draw (1) -- (2);
    \draw (2) -- (3);
    \draw[->] (3) -- (4);
    \draw[->] (5) -- (4);
    \draw[->] (6) -- (5);
    \draw[->] (6) -- (7);
    \draw[->] (7) -- (8);
    \draw[->] (8) -- (9);
    \draw[->] (8) -- (10);
    \draw (11) -- (12);
    \draw (12) -- (13);
    \draw (12) -- (14);
    \draw (11) -- (15);
    \draw[->] (16) -- (18);
    \draw[->] (17) -- (18);
    \draw (19) -- (20);
    \draw (20) -- (21);
    \end{tikzpicture} }\\
\subfloat[MPDAG $\mathcal{G}$ with the root assumption.]{
    \label{fig: real_data_with_assumption}
    \begin{tikzpicture}[node distance={12mm}, thick, main/.style = {draw, rectangle}]
    \tikzstyle{every node}=[font=\tiny]
    \node[main] (3) {$Mjob$}; 
    \node[main] (2) [below of=3] {$Medu$};
    \node[main] (1) [below of=2]{$Fedu$}; 
    \node[main] (4) [right of=2] {$Grade$};
    \node[main] (5) [above of=4] {$schoolsup$}; 
    \node[main] (6) [above right of=5] {$age$};
    \node[main] (7) [below right of=6] {$school$};
    \node[main] (8) [below of=7] {$address$};
    \node[main] (9) [below left of=8] {$internet$}; 
    \node[main] (10) [below right of=8] {$traveltime$};
    \node[main] (13) [right of=8] {$goout$};
    \node[main] (12) [above of=13] {$Walc$};
    \node[main] (14) [below right of=12] {$Dalc$};
    \node[main] (11) [above right of=12] {$sex$}; 
    \node[main] (15) [right of=14] {$studytime$};
    \node[main] (16) [below of=14] {$Fjob$};
    \node[main] (18) [below right of=16] {$absences$};
    \node[main] (17) [above right of=18] {$reason$};
    \node[main] (19) [below left of=9] {$famsup$};
    \node[main] (20) [right of=19] {$paid$};
    \node[main] (21) [right of=20] {$higher$};
    \node[main] (24) [below left of=21] {$activities$};
    \node[main] (25) [below right of=21] {$famsize$};
    \node[main] (27) [below left of=19] {$guardian$};
    \node[main] (23) [below left of=24] {$Pstatus$};
    \node[main] (22) [left of=23] {$famrel$};
    \node[main] (29) [below right of=25] {$nursery$};
    \node[main] (28) [right of=29] {$health$};
    \node[main] (26) [below right of=24] {$freetime$};
    \node[main] (30) [below right of=18] {$romantic$};
    \draw (1) -- (2);
    \draw (2) -- (3);
    \draw[->] (3) -- (4);
    \draw[->] (5) -- (4);
    \draw[->] (6) -- (5);
    \draw[->] (6) -- (7);
    \draw[->] (7) -- (8);
    \draw[->] (8) -- (9);
    \draw[->] (8) -- (10);
    \draw[->] (11) -- (12);
    \draw[->] (12) -- (13);
    \draw[->] (12) -- (14);
    \draw[->] (11) -- (15);
    \draw[->] (16) -- (18);
    \draw[->] (17) -- (18);
    \draw (19) -- (20);
    \draw (20) -- (21);
    \end{tikzpicture} 
}
\subfloat[An arbitrary DAG $\mathcal{D}$ in Markov equivalence class of $\mathcal{G}$.]{ 
    \label{fig: real_data_DAG}
    \begin{tikzpicture}[node distance={12mm}, thick, main/.style = {draw, rectangle}]
    \tikzstyle{every node}=[font=\tiny]
    \node[main] (3) {$Mjob$}; 
    \node[main] (2) [below of=3] {$Medu$};
    \node[main] (1) [below of=2]{$Fedu$}; 
    \node[main] (4) [right of=2] {$Grade$};
    \node[main] (5) [above of=4] {$schoolsup$}; 
    \node[main] (6) [above right of=5] {$age$};
    \node[main] (7) [below right of=6] {$school$};
    \node[main] (8) [below of=7] {$address$};
    \node[main] (9) [below left of=8] {$internet$}; 
    \node[main] (10) [below right of=8] {$traveltime$};
    \node[main] (13) [right of=8] {$goout$};
    \node[main] (12) [above of=13] {$Walc$};
    \node[main] (14) [below right of=12] {$Dalc$};
    \node[main] (11) [above right of=12] {$sex$}; 
    \node[main] (15) [right of=14] {$studytime$};
    \node[main] (16) [below of=14] {$Fjob$};
    \node[main] (18) [below right of=16] {$absences$};
    \node[main] (17) [above right of=18] {$reason$};
    \node[main] (19) [below left of=9] {$famsup$};
    \node[main] (20) [right of=19] {$paid$};
    \node[main] (21) [right of=20] {$higher$};
    \node[main] (24) [below left of=21] {$activities$};
    \node[main] (25) [below right of=21] {$famsize$};
    \node[main] (27) [below left of=19] {$guardian$};
    \node[main] (23) [below left of=24] {$Pstatus$};
    \node[main] (22) [left of=23] {$famrel$};
    \node[main] (29) [below right of=25] {$nursery$};
    \node[main] (28) [right of=29] {$health$};
    \node[main] (26) [below right of=24] {$freetime$};
    \node[main] (30) [below right of=18] {$romantic$};
    \draw[->] (1) -- (2);
    \draw[->] (2) -- (3);
    \draw[->] (3) -- (4);
    \draw[->] (5) -- (4);
    \draw[->] (6) -- (5);
    \draw[->] (6) -- (7);
    \draw[->] (7) -- (8);
    \draw[->] (8) -- (9);
    \draw[->] (8) -- (10);
    \draw[->] (11) -- (12);
    \draw[->] (12) -- (13);
    \draw[->] (12) -- (14);
    \draw[->] (11) -- (15);
    \draw[->] (16) -- (18);
    \draw[->] (17) -- (18);
    \draw[->] (19) -- (20);
    \draw[->] (20) -- (21);
    \end{tikzpicture} }
\caption{The causal graphs for Student dataset. The attribute information can be found at \url{https://archive.ics.uci.edu/ml/datasets/Student+Performance}. (a) is the learnt CPDAG $\mathcal{C}$; (b) Given the background knowledge that the age is the parent of \textit{schoolsup} and \textit{school}, without any other assumption, we can have the MPDAG $\mathcal{G'}$ by applying Meek's rule; (c) With the additional root node assumption, we can obtain the MPDAG $\mathcal{G}$; (d) is an arbitary DAG $\mathcal{D}$ in the Markov equivalent class of $\mathcal{G}$.}.
\end{figure}

\subsection{Training details on real data} \label{Appendix: Training Details on Real Data}

We assume the linear causal model in the obtained MPDAG $\mathcal{G}$. To test the counterfactual fairness of the baseline methods, as in \cref{{sec: Experiment/Synthetic Data}}, we first generate the counterfactual data. Since the ground-truth DAG is unknown, we generate the counterfactual data from a DAG sampled from the Markov equivalence class of MPDAG $\mathcal{G}$.\footnote{On this dataset, all the possible true DAGs give the same results as the nodes with uncertain edge directions are not related to the sensitive attribute.} Then we fit the parameters of the model using the original data and generate samples from the model given the counterfactual \textit{sex} and the same noise in the original data for each individual. We fit baseline models to both the original and counterfactual sampled data and measure the unfairness in the same way as in \cref{sec: Experiment/Synthetic Data}. This procedure is carried out $10$ times and the average unfairness and RMSE results for five models are reported in \cref{tab: RMSE and Unfairness for real data}.

\subsection{Experiment based on more complicated structural equations} \label{Appendix: Experiment based on non-linear structural equations}

To show the generality of our method, we generate each variable $X_i$ from the following non-linear structure equation:
\begin{equation} \label{eq: nonlinear structual equation model}
    X_i = g_i(f_i(pa(X_i)+\epsilon_i)), i=1,...,n,
\end{equation}
where the causal mechanism $f_i$ is randomly chosen from \textit{linear}, \textit{sin}, \textit{cos}, \textit{tanh}, \textit{sigmoid} function and their combinations; $g_i$, which represents the post-nonlinear distortion in variable $X_i$, is randomly chosen from \textit{linear} function, \textit{absolute} and \textit{reciprocal} function; $\epsilon_i$ is the noise term, sampling from \textit{Gaussian}, \textit{Exponential} and \textit{Gumbel} distributions. With the same basic settings and evaluation metrics as \cref{sec: Experiment/Synthetic Data}, we fit a SVM regression model for the baselines and our proposed models. The average unfairness and RMSE achieved on 100 causal graphs is reported in \cref{tab: RMSE and Unfairness for simultion data_non-linear structural equations}. The corresponding boxplot is shown in \cref{fig: Boxplot_unfairness_and_RMSE_nonlinear} as well. We can see that it yields the same trend on counterfactual fairness as the linear case, while the accuracy in this dataset does not necessarily decrease with the increase in fairness. More discussion on accuracy-fairness tradeoff can be referred to \cref{Appendix: Discussion on Accuracy-fairness trade-off}.

\begin{table}[!ht]
\caption{Average unfairness and RMSE for synthetic datasets generated by nonlinear structural equations on held-out test set. For each graph setting, the unfairness gets decreasing from left to right, while there is no obvious increase in RMSE.}
\label{tab: RMSE and Unfairness for simultion data_non-linear structural equations}
\begin{center}
\small
\begin{tabular}{cccccccc}
    \hline
           &Node &Edge   &Full     &Unaware  &FairRelax  &Oracle   &Fair\\ \hline
  \multirow{4}{*}{\begin{turn}{90}Unfairness\end{turn}}  &$10$   &$20$    &$0.575\pm0.431$ &$0.218\pm0.262$  &$0.028\pm0.115$ &$0.000\pm0.000$ &$0.000\pm0.000$ \\ 
            &$20$   &$40$    &$0.491\pm0.358$ &$0.143\pm0.210$ &$0.017\pm0.080$ &$0.000\pm0.000$ &$0.000\pm0.000$ \\ 
            &$30$   &$60$    &$0.388\pm0.309$ &$0.126\pm0.208$ &$0.010\pm0.044$ &$0.000\pm0.000$ &$0.000\pm0.000$ \\ 
            &$40$   &$80$    &$0.388\pm0.384$ &$0.094\pm0.139$ &$0.009\pm0.057$ &$0.000\pm0.000$ &$0.000\pm0.000$ \\ \hline\hline
    
    \multirow{4}{*}{\begin{turn}{90}RMSE\end{turn}}    &$10$   &$20$    &$4.033\pm4.675$   &$4.024\pm4.663$  &$4.095\pm4.638$   &$4.098\pm4.649$    &$4.101\pm4.646$ \\ 
            &$20$   &$40$    &$3.921\pm4.532$   &$3.881\pm4.467$   &$3.921\pm4.497$      &$3.920\pm4.495$   &$3.927\pm4.491$ \\ 
            &$30$   &$60$    &$3.370\pm3.960$    &$3.371\pm3.958$  &$3.437\pm4.024$    &$3.438\pm4.025$   &$3.442\pm4.023$\\ 
            &$40$   &$80$    &$3.457\pm3.999$     &$3.451\pm3.983$ &$3.474\pm3.956$    &$3.478\pm3.960$ &$3.479\pm3.963$ \\ \hline
\end{tabular}
\end{center}
\end{table}

\begin{figure}[ht]
\centering
\subfloat[Average unfairness for each model and graph setting.]{
\label{fig: Boxplot_unfairness_nonlinear}
\includegraphics[width=0.48\columnwidth,height=0.3\columnwidth]{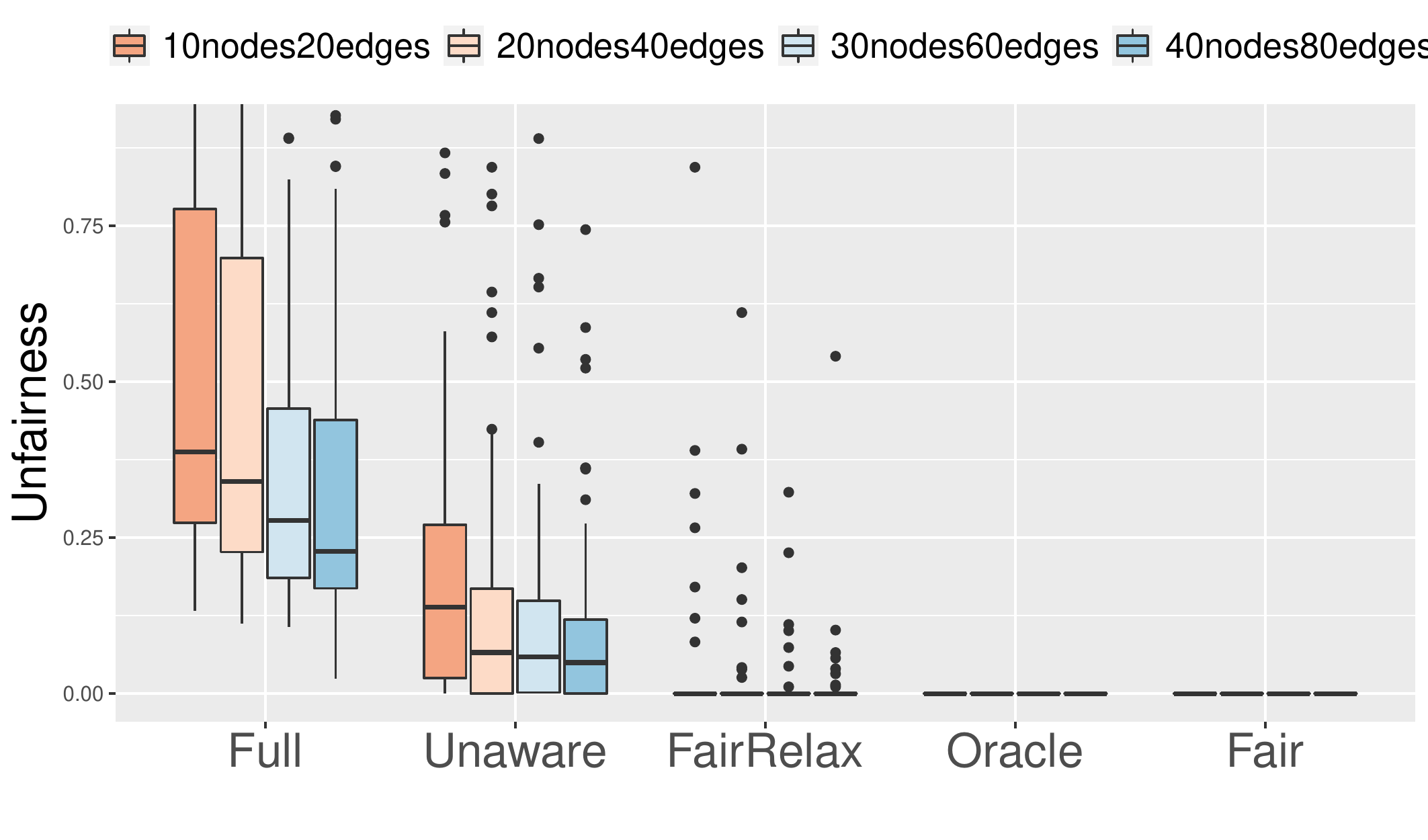}
}
\subfloat[Average RMSE for each model and graph setting.]{
\label{fig: Boxplot_RMSE_nonlinear}
\includegraphics[width=0.48\columnwidth,height=0.3\columnwidth]{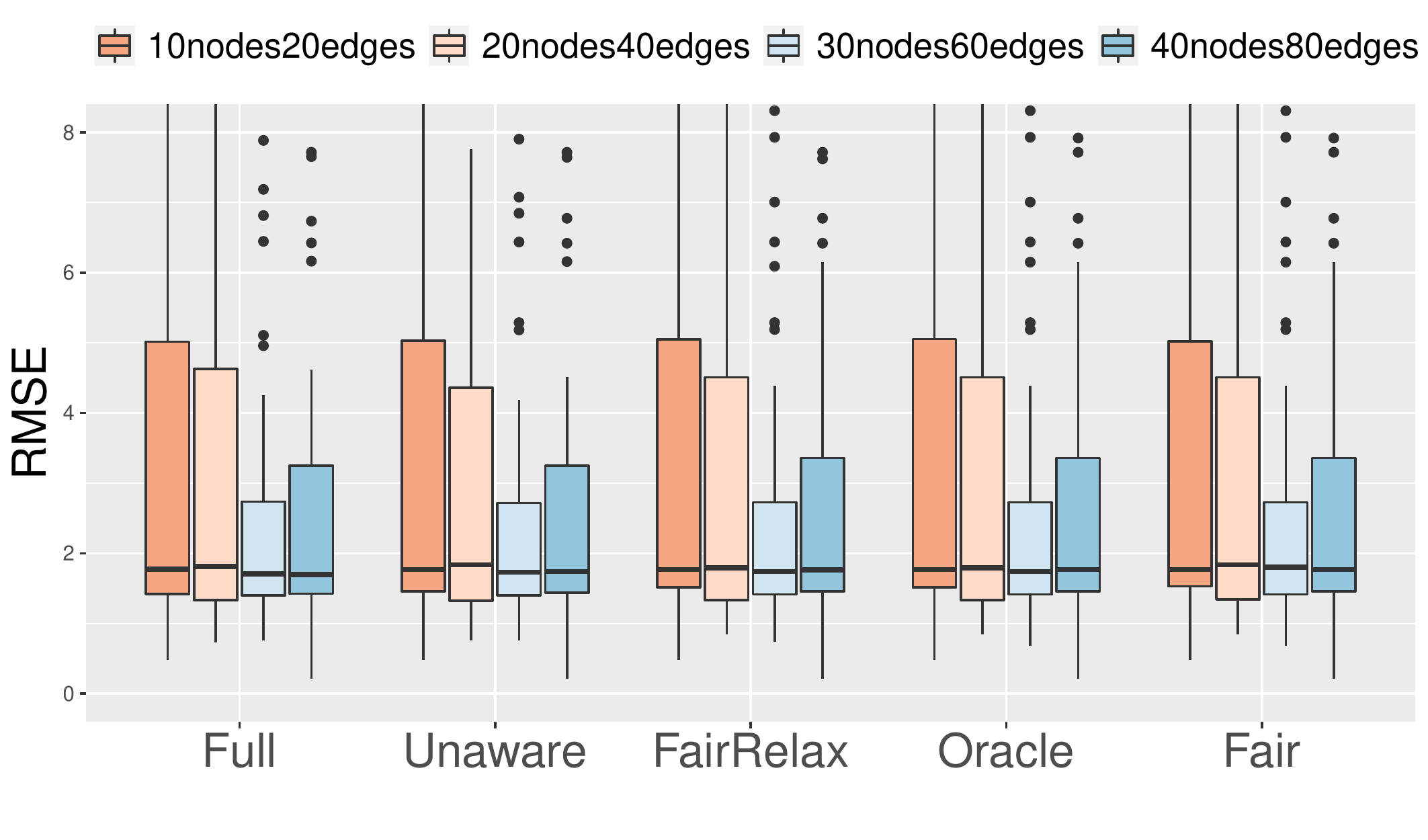}
}
\caption[]{Average unfairness and RMSE for synthetic datasets generated by nonlinear structural equations on held-out test set.}
\label{fig: Boxplot_unfairness_and_RMSE_nonlinear}
\end{figure}

\subsection{Experiment analyzing fairness performance with varying amount of given domain knowledge} \label{Appendix: Experiment analyzing fairness performance with varying amount of given domain knowledge}

Prediction using \texttt{Fair} method results in a strictly fair model regardless of how much domain knowledge is given. Prediction using \texttt{FairRelax} method will show different fairness performance with different amount of domain knowledge, since some definite descendants $X$ of the sensitive attribute may be possible descendants when less domain knowledge is given, thus the unfair feature $X$ will be involved to make predictions. At this point we do not know, theoretically, how different amounts or even types of domain knowledge will affect the performance of \texttt{FairRelax}. However, we can explore this experimentally. 

For a given CPDAG, only the fairness performance of \texttt{FairRelax} model among all models will be affected by the amount of background knowledge. The more background knowledge, the fairer the \texttt{FairRelax}. For example, in the setting `10nodes20edges', when the proportion of the undirected edges' true orientation is increased from 0.1, 0.3, 0.6 to 1, the unfairness for \texttt{FairRelax} are 0.075, 0.023, 0.018, 0.0, respectively. The same trend can be found in other graph settings, see \cref{tab: Fairness performance with varying amount of given domain knowledge}. The ` BK (\%)' represents how much the undirected edges' true orientation is given as background knowledge.

\begin{table}[!ht]
\caption{Average unfairness for synthetic datasets with varying amount of given domain knowledge. For each graph setting, the more domain knowledge, the fairer the model \texttt{FairRelax} becomes.}
\label{tab: Fairness performance with varying amount of given domain knowledge}
\begin{center}
\small
\begin{tabular}{cccccccc}
    \hline
           Node  &Edge  &BK(\%)   &Full     &Unaware  &FairRelax  &Oracle   &Fair\\ \hline
  \multirow{4}{*}{10} &\multirow{4}{*}{20} &$10$    &$0.707\pm1.144$ &$0.587\pm1.093$  &$0.075\pm0.334$ &$0.0\pm0.0$ &$0.0\pm0.0$ \\ 
        &    &$30$    &$0.707\pm1.144$ &$0.587\pm1.093$ &$0.023\pm0.178$ &$0.0\pm0.0$ &$0.0\pm0.0$ \\ 
        &    &$60$    &$0.707\pm1.144$ &$0.587\pm1.093$ &$0.018\pm0.174$ &$0.0\pm0.0$ &$0.0\pm0.0$ \\
        &    &$100$   &$0.707\pm1.144$ &$0.587\pm1.093$ &$0.000\pm0.174$ &$0.0\pm0.0$ &$0.0\pm0.0$ \\ \hline\hline
        
  \multirow{4}{*}{20} &\multirow{4}{*}{40} &$10$    &$0.326\pm0.640$ &$0.280\pm0.624$  &$0.032\pm0.189$ &$0.0\pm0.0$ &$0.0\pm0.0$ \\ 
       &     &$30$   &$0.326\pm0.640$ &$0.280\pm0.624$ &$0.018\pm0.136$ &$0.0\pm0.0$ &$0.0\pm0.0$\\ 
       &     &$60$   &$0.326\pm0.640$ &$0.280\pm0.624$ &$0.014\pm0.132$ &$0.0\pm0.0$ &$0.0\pm0.0$\\
       &     &$100$  &$0.326\pm0.640$ &$0.280\pm0.624$ &$0.000\pm0.000$ &$0.0\pm0.0$ &$0.0\pm0.0$\\ \hline\hline
    
  \multirow{4}{*}{30} &\multirow{4}{*}{60}  &$10$    &$0.442\pm1.176$ &$0.433\pm1.191$  &$0.080\pm0.329$ &$0.0\pm0.0$ &$0.0\pm0.0$ \\ 
      &      &$30$   &$0.442\pm1.176$ &$0.433\pm1.191$ &$0.076\pm0.321$ &$0.0\pm0.0$ &$0.0\pm0.0$\\ 
      &     &$60$   &$0.442\pm1.176$ &$0.433\pm1.191$ &$0.056\pm0.307$ &$0.0\pm0.0$ &$0.0\pm0.0$\\
      &     &$100$  &$0.442\pm1.176$ &$0.433\pm1.191$ &$0.000\pm0.000$ &$0.0\pm0.0$ &$0.0\pm0.0$\\ \hline\hline

  \multirow{4}{*}{40} &\multirow{4}{*}{80}  &$10$    &$0.221\pm0.646$ &$0.199\pm0.647$  &$0.042\pm0.220$ &$0.0\pm0.0$ &$0.0\pm0.0$ \\ 
     &       &$30$   &$0.221\pm0.646$ &$0.199\pm0.647$ &$0.019\pm0.176$ &$0.0\pm0.0$ &$0.0\pm0.0$\\ 
     &       &$60$   &$0.221\pm0.646$ &$0.199\pm0.647$ &$0.001\pm0.010$ &$0.0\pm0.0$ &$0.0\pm0.0$\\
     &       &$100$  &$0.221\pm0.646$ &$0.199\pm0.647$ &$0.000\pm0.000$ &$0.0\pm0.0$ &$0.0\pm0.0$\\ \hline
    
\end{tabular}
\end{center}
\end{table}

\subsection{Experiment analyzing model robustness on causal discovery algorithms} \label{Appendix: Experiment analyzing method robustness on causal discovery algorithm}

In \cref{sec: Experiment/Synthetic Data}, we obtain the ture CPDAG from the true DAG without running the causal discovery algorithm. However, in practice, with the true DAG unknown, the CPDAG can only be obtained from causal discovery algorithms. In order to test the model robustness on causal discovery algorithms, we learn the corresponding CPDAG from the synthetic data by the Greedy Equivalence Search (GES) procedure \citep{chickering2002learning}. The prediction performance and fairness results are reported in \cref{tab: RMSE and Unfairness for simultion data_analyzing method robustness on GES} and \cref{fig: Boxplot_unfairness_and_RMSE_GES}, from which, we can see the same trend on five models as the one in \cref{sec: Experiment/Synthetic Data}. Moreover, there is not much difference on fairness and prediction performance on \texttt{FairRelax} model between the case that the CPDAG is induced directly from the true DAG and the case that the CPDAG is learnt from the observational data by a causal discovery algorithm.

\begin{table}[!ht]
\caption{Average unfairness and RMSE for synthetic datasets on held-out test set when the corresponding CPDAG is learned by GES search procedure. For each graph setting, the unfairness gets decreasing from left to right and the RMSE gets increasing from left to right.}
\label{tab: RMSE and Unfairness for simultion data_analyzing method robustness on GES}
\begin{center}
\small
\begin{tabular}{cccccccc}
    \hline
           &Node &Edge   &Full     &Unaware  &FairRelax  &Oracle   &Fair\\ \hline
  \multirow{4}{*}{\begin{turn}{90}Unfairness\end{turn}}  &$10$   &$20$    &$0.264\pm0.343$ &$0.203\pm0.318$  &$0.084\pm0.215$ &$0.000\pm0.000$ &$0.079\pm0.215$ \\ 
            &$20$   &$40$    &$0.191\pm0.312$ &$0.150\pm0.283$ &$0.067\pm0.243$ &$0.000\pm0.000$ &$0.066\pm0.243$ \\ 
            &$30$   &$60$    &$0.157\pm0.301$ &$0.143\pm0.308$ &$0.066\pm0.219$ &$0.000\pm0.000$ &$0.061\pm0.216$ \\ 
            &$40$   &$80$    &$0.096\pm0.190$ &$0.074\pm0.183$ &$0.038\pm0.109$ &$0.000\pm0.000$ &$0.024\pm0.075$ \\ \hline\hline
    
    \multirow{4}{*}{\begin{turn}{90}RMSE\end{turn}}    &$10$   &$20$    &$0.616\pm0.255$   &$0.631\pm0.262$  &$1.071\pm0.739$   &$1.079\pm0.788$    &$1.112\pm0.767$ \\ 
            &$20$   &$40$    &$0.597\pm0.252$   &$0.601\pm0.250$   &$1.029\pm0.736$   &$0.862\pm0.564$    &$1.037\pm0.734$ \\ 
            &$30$   &$60$    &$0.592\pm0.232$    &$0.595\pm0.235$  &$0.992\pm0.771$    &$0.907\pm0.894$   &$1.097\pm0.955$\\ 
            &$40$   &$80$    &$0.595\pm0.273$     &$0.596\pm0.272$ &$0.928\pm0.738$    &$0.746\pm0.433$ &$0.947\pm0.753$ \\ \hline
\end{tabular}
\end{center}
\end{table}

\begin{figure}[ht]
\centering
\subfloat[Average unfairness for each model and graph setting.]{
\label{fig: Boxplot_unfairness_GES}
\includegraphics[width=0.48\columnwidth,height=0.3\columnwidth]{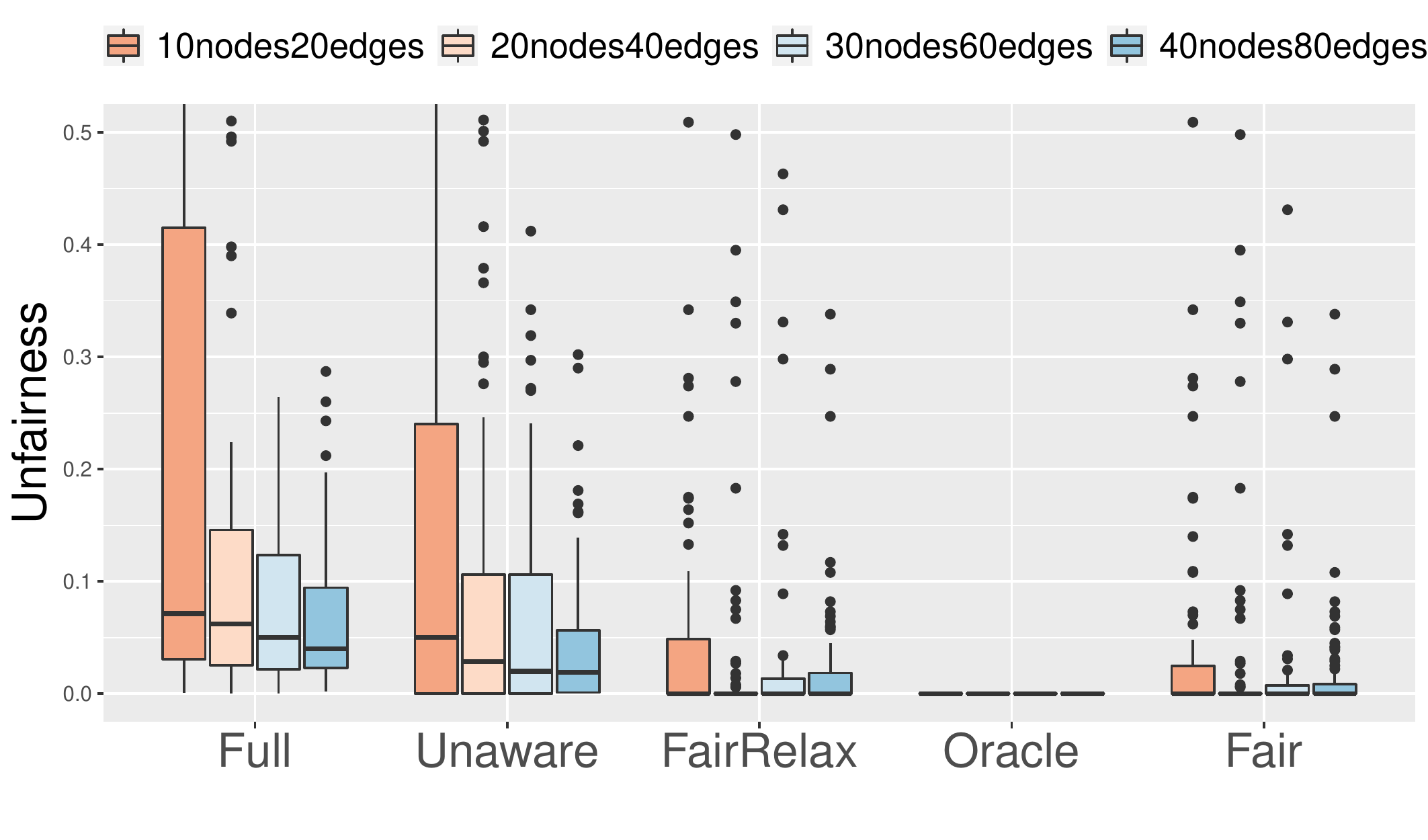}
}
\subfloat[Average RMSE for each model and graph setting.]{
\label{fig: Boxplot_RMSE_GES}
\includegraphics[width=0.48\columnwidth,height=0.3\columnwidth]{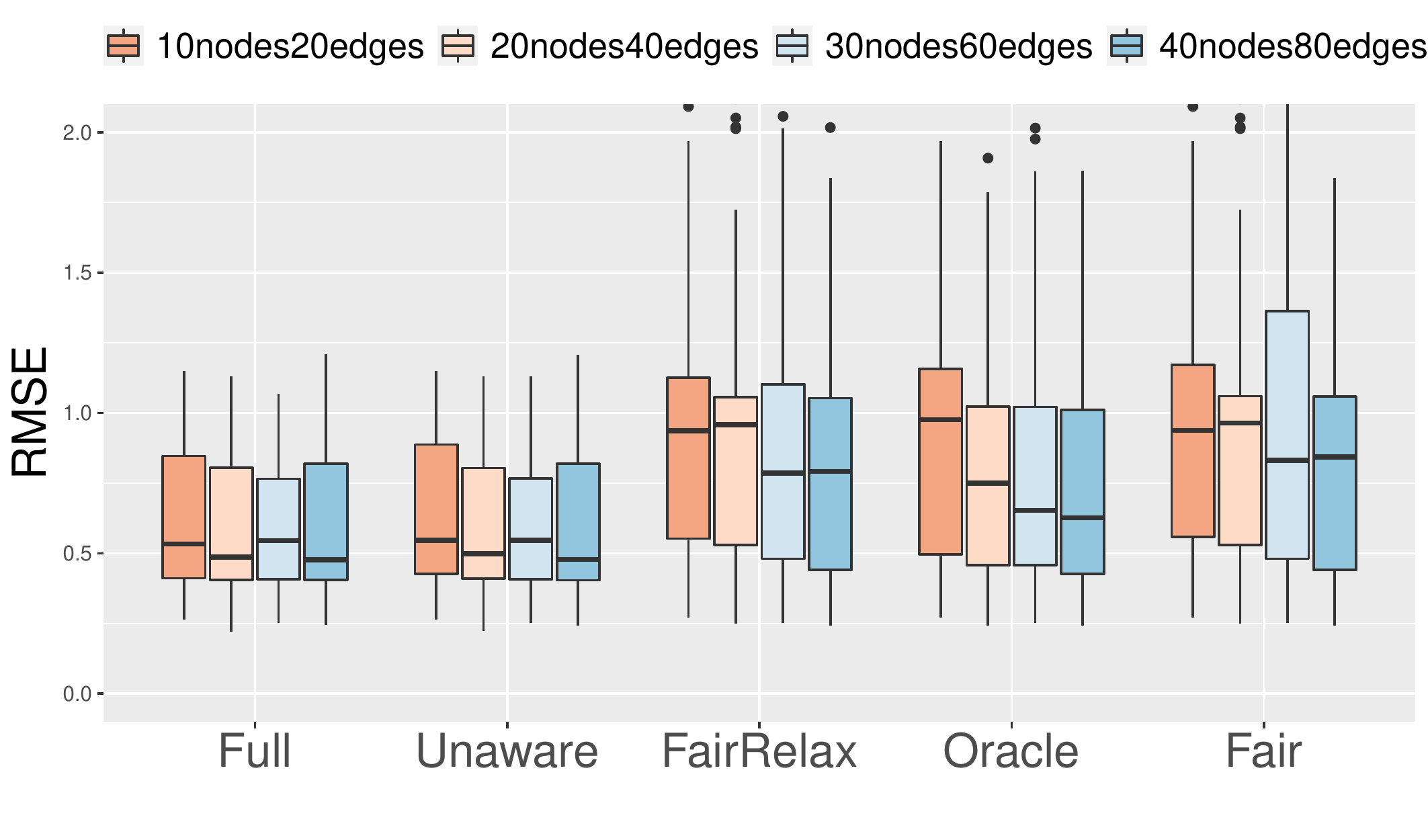}
}
\caption[]{Average unfairness and RMSE for synthetic datasets on held-out test set when the corresponding CPDAG is learned by GES search procedure. For each graph setting, the unfairness gets decreasing from left to right, while RMSE has the opposite trend.}
\label{fig: Boxplot_unfairness_and_RMSE_GES}
\end{figure}

%% file: Pages/Appendices/RelatedWork.tex
\section{Additional related works on ancestral relations identifiability}\label{sec: Related Work}

A basic task in causal reasoning on an MPDAG $\mathcal{G}$ is to identify the ancestral relations between two distinct nodes in $\mathcal{G}$. The first intuitive method is to list all DAGs in $\mathcal{G}$ and then read off the ancestral relations in each DAG. However, this method is computationally burdensome. 
The second approach is to measure the possible causal effect \citep{perkovic2017interpreting, guo2021minimal, perkovic2020identifying, fang2020ida, liu2020local, liu2020collapsible} from the source variable to the target variable in an MPDAG. The target is a definite descendant (or non-descendant) of the source if and only if all possible causal effect are non-zero (or zero).
The third approach is to analyse the path from the source to target in an MPDAG $\mathcal{G}$.  Perkovi{\'c} et al. \citep{perkovic2017interpreting} propose that the target is a definite non-descendant of the source if and only if there is no b-possibly causal path from the source to target. There is also a sufficient and necessary graphical condition \cite[Theorem 1]{fang2022local} to identify whether a variable is a definite descendant of another variable in CPDAGs. The authors in \citep{roumpelaki2016marginal, mooij2020constraint} \footnote{Although Theorem 3.1 in \citep{roumpelaki2016marginal} proved the necessity, their proof is incomplete as mentioned by \citep{mooij2020constraint}. Proving the necessity for more general types causal graphs remains an open problem \citep{zhang2006causal}.} extend the sufficiency of this condition to other kinds of causal graphs as well. However, to the best of our knowledge, such graphical criterion to determine the definite descendants for MPDAGs has not been examined before.

%% file: Pages/Appendices/Discussion.tex
\section{Discussion on accuracy-fairness trade-off} \label{Appendix: Discussion on Accuracy-fairness trade-off}

The accuracy-fairness trade-off is pointed out in a great number of existing algorithmic fairness works \citep{martinez2019fairness, zhao2019inherent, menon2018cost, zliobaite2015relation, chen2018my, wei2022}. Yet, accuracy may not be doomed to decrease as fairness increases depending on the data setting \citep{dutta2020there,  friedler2021possibility, Yeom2018DiscriminativeBN}. For example, in the synthetic dataset in \cref{sec: Experiment/Synthetic Data}, we do happen to observe an accuracy-fairness trade-off, while it seems that such trade-off does not exist in the synthetic dataset generated by the nonlinear structure equations in \cref{Appendix: Experiment based on non-linear structural equations}. The authors in \citep{wick2019unlocking, inproceedings, dutta2020there} describe when such trade-off exists and when it does not theoretically or empirically. 
Future work may take the fairness-accuracy trade-off into more consideration. 